\documentclass{article}

\PassOptionsToPackage{numbers, sort&compress}{natbib}
\usepackage[preprint]{neurips_2026}

\usepackage[utf8]{inputenc} 
\usepackage[T1]{fontenc}    
\usepackage{hyperref}       
\usepackage{url}            
\usepackage{booktabs}       
\usepackage{amsfonts}       
\usepackage{nicefrac}       
\usepackage{microtype}      
\usepackage{xcolor}  
\usepackage{amsmath}
\usepackage{amsthm}
\DeclareMathOperator*{\argmax}{arg\,max}
\usepackage{enumitem}
\usepackage{algorithm}
\usepackage{caption}
\usepackage{algpseudocode}
\usepackage{graphicx}
\usepackage{subcaption}
\usepackage{multirow}

\newtheorem{theorem}{Theorem}
\newtheorem{lemma}{Lemma}

\title{Valid Best-Model Identification for LLM Evaluation via Low-Rank Factorization}

\author{Elad Tolochinsky$^{1}$ ~~~~~ Yaniv Tenzer$^{1}$ ~~~~~ Yaniv Romano$^{1,2}$\\ \\
$^1$Department of Computer Science, Technion -- Israel Institute of Technology\\
$^2$Department of Electrical and Computer Engineering, Technion -- Israel Institute of Technology\\
\small{\texttt{elad.t@cs.technion.ac.il},\hspace{0.05cm} \texttt{yanivt@technion.ac.il},\hspace{0.05cm} \texttt{yromano@technion.ac.il}}
}

\begin{document}

\maketitle

\begin{abstract}

Selecting the best large language model (LLM) for a fixed benchmark is often expensive, since exhaustive evaluation requires running every model on every example. Multi-armed bandit (MAB) algorithms can reduce the number of LLM calls by sequentially selecting the next model–example pair to evaluate, thereby avoiding wasted evaluations on clearly underperforming models. Further savings can be achieved by predicting model scores from the partially observed model–example score matrix using low-rank factorization. However, such predictions are not ground truth: they can be biased and may therefore lead to incorrect identification of the best model. In this work, we propose a principled framework that combines MAB with cheap predicted scores without compromising statistical validity. Specifically, we derive doubly robust estimators of each model's performance that use the low-rank predictions to reduce variance. This enables the construction of valid finite-sample confidence intervals in our setting, where models are selected adaptively and examples are sampled without replacement. Empirical results on real-world benchmarks show that our approach reduces the number of required evaluations, yielding meaningful savings in compute and cost while accurately identifying the best-performing model.
\end{abstract}

\section{Introduction}
\label{sec:introduction}

LLMs are employed across a wide range of applications, including customer-facing
dialogue systems, code generation, writing assistance, and information retrieval
~\cite{bommasani2021opportunities,jiang2026survey,wei2022emergent}. As LLM adoption grows, practitioners often encounter the challenge of choosing the best-performing model and configuration---such as prompts and hyperparameters---for a specific task. A common practice for that purpose is to compare the performances of the models on fixed benchmarks using a score function. For example, consider the task of mathematical problem solving, as studied in benchmarks such as MMLU~\cite{wang2024mmlu} and MATH~\cite{hendrycksmath2021}. In this case, a score function can be a binary function that indicates whether the model’s solution is correct. Unfortunately, exhaustively evaluating every model on every example is resource-intensive: a single evaluation on the GAIA benchmark~\cite{ICLR2024_25ae35b5} can cost up to \$2,829~\cite{ndzomga2026efficient}, and the recent Holistic Agent Leaderboard evaluation spent roughly \$40,000 to compare 9 models across 9 benchmarks~\cite{kapoor2025hal}. Moreover, these are single-run figures: stochastic agents require repeated rollouts to produce reliable estimates, multiplying costs further.

To reduce the evaluation cost, recent works utilize multi-armed bandit (MAB) algorithms, such as UCB-E \cite{audibert2010best}, to adaptively allocate the limited evaluation budget. Ideally, MAB should spend more budget on promising candidates and less on clearly underperforming ones. In this framework, each model is an arm, and pulling an arm corresponds to evaluating the model on examples and observing the scores. The evaluation proceeds sequentially: at each step, the algorithm selects a model and a batch of examples that this model has not yet evaluated, computes their scores, and updates its estimate of the model's mean accuracy. These updated estimates are then used to guide the selection of the next model.

Further computational savings can be achieved by using cheap predicted scores rather than the exact ones. Such predictions can be found in the structure of the model-example score matrix, i.e., the matrix whose \((i,j)\) entry is the score obtained by running the \(i\)-th model on the \(j\)-th example. Concretely, the work in \cite{zhou2024speeding} utilizes the fact that models may exhibit similar performance profiles across examples, and examples may admit similar scores across models. This correlation suggests that the full score matrix may be well approximated by a low-rank matrix. During evaluation, this matrix is only partially observed. Yet, low-rank factorization provides a cheap way to impute the missing entries, producing predicted scores that can guide MAB's evaluation procedure. Indeed, it is shown in \cite{zhou2024speeding} that such predictions can accelerate the identification of the best model.

The challenge is that these predicted scores are not ground truth. They can be
biased and may therefore lead to incorrect identification of the best model.
This raises the need for a principled framework that incorporates predicted
scores while maintaining unbiased estimates of the models' mean performance.
To this end, we build on doubly robust estimators---including augmented inverse probability weighting (AIPW)~\cite{robins1995semiparametric,tsiatis2006semiparametric} and prediction-powered inference (PPI)~\cite{angelopoulos2023prediction, angelopoulos2023ppi++}---that enable
the construction of unbiased estimators of population parameters with
potentially lower variance by leveraging machine learning (ML) predictions. These tools have been developed for adaptive MAB, contextual bandits, and even
best-arm identification~\citep{hadad2021confidence,dudik2011doubly,
wang2017optimal,dimakopoulou2021online}.
However, it is unclear how to extend the existing guarantees for these techniques to the finite-population setting considered in this paper. Beyond sampling examples
without replacement from a fixed benchmark, our setting introduces an additional complication: the
imperfect predictions are obtained by repeatedly refitting an online low-rank
model to the same adaptively observed model--example score matrix, resulting in strong cross-arm dependence.

In this work, we bridge the gap by presenting PULSE: Prediction-Powered Unbiased Low-Rank Sequential Evaluation. Given a fixed budget of LLM calls, PULSE leverages imperfect predicted scores to increase the probability of identifying the best-performing model. Importantly, our approach remains statistically valid even when the predictions are inaccurate. As we show, when predictions are informative, our method accelerates identification; when they are not, it can approach standard MAB behavior that relies solely on observed evaluations.

\begin{minipage}[c]{0.53\textwidth}
Our main contributions are as follows:
\begin{itemize}[leftmargin=*]
    \item Our first key contribution (Section~\ref{sec: the method}) is an
    estimator of each model's mean accuracy that leverages predictions obtained
    by low-rank factorization of the score matrix. The resulting estimator is
    guaranteed to be unbiased, while informative predictions can reduce its
    variance. We also show how to construct confidence intervals (CIs) for the
    models' mean accuracies and establish their finite-sample validity.

    \item Building on our theoretical results, our second key contribution
    (Section~\ref{sec: the method}) is a MAB algorithm for best-model
    identification. We establish a finite-sample success guarantee for PULSE and analyze how prediction quality can affect its convergence.

    \item Finally, we evaluate the performance of our algorithm on 2.2K models
    and six real-world benchmarks. As shown in
    Figure~\ref{fig:budget_to_95_bars}, our experiments demonstrate that PULSE
    can save up to $46\%$ of LLM calls relative to UCB-E when targeting $95\%$
    identification accuracy. Code: \url{https://github.com/elad-tolo/pulse-llm-eval}.
\end{itemize}
\end{minipage}\hfill
\begin{minipage}[c]{0.43\textwidth}
\centering
\captionsetup{type=figure}
\includegraphics[width=\linewidth]{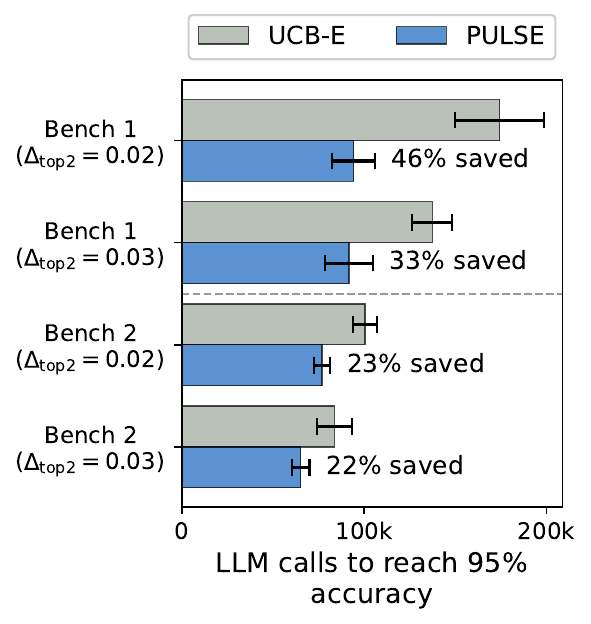}
\captionof{figure}{Budget needed for 95\% best-model identification accuracy (lower is better). Bench 1 and 2 cover 2.2K models each. $\Delta_{\text{top2}}$, the gap between the top two models, controls difficulty. See Section~\ref{sec:experiments}.}
\label{fig:budget_to_95_bars}
\end{minipage}

\section{Problem Formulation and Related Work}
We let $\mathcal{X} = \{x_1,\ldots,x_n\}$ be a fixed dataset of examples, $\mathcal{M} = \{f_1,\ldots,f_m\}$ be a set of LLMs. Let $s(M(x)) \in \{0,1\}$ be a scoring function used to score the output of the LLMs. Without loss of generality, we assume that
$s(f(x)) > s(f'(x))$ indicates that model \(f\) has better performance than model \(f'\) on example \(x\). We denote by $S \in \{0,1\}^{m \times n}$ the score matrix so that $S_{ij} := s(f_i(x_j))$. In our setting, $S$ is only partially observed, and we think of its entries as random variables. More precisely, let $S_i$ be the score obtained by model $i$ when tested on an example that is drawn uniformly from $\mathcal{X}$. Under this formulation, $\mu_i = \mathbb{E}[S_i] = \frac{1}{n}\sum_{j =1}^{n}S_{i,j}$, where the expectation is taken with respect to the random choice of the example. In contrast with an infinite population setting, in which the sample average is only an estimate for the population mean, here the sample average is the actual expectation of $S^i$.

Given a budget of \(T\) LLM calls, we consider an adaptive algorithm
\(\mathcal{A}(T, \mathcal{M}, \mathcal{X})\) that returns an index
\(\hat i\in[m]\), where \([m] = \{1, \dots, m\}\). Our goal is to maximize the probability of correctly identifying the best model \(i^* = \arg\max_{i\in[m]} \mu_i\), which we assume is unique. We denote this probability by \(P(\hat i=i^*)\), taken over the randomness of the adaptive sampling procedure. Our algorithm follows the UCB-E method~\cite{audibert2010best}. At each round \(t\), the algorithm maintains CIs for the mean accuracy of each model and selects the model with the largest upper confidence bound. It then evaluates the selected model on a batch of previously unseen examples. Using the newly observed scores, the algorithm updates the estimate of the selected model’s mean accuracy and its corresponding CI.

The convergence rate of the UCB-E algorithm to the optimal arm (best model) depends on the width of the CIs~\cite{audibert2010best}, which decreases with the number of observations. This motivates us to artificially increase the number of samples by predicting the unobserved entries of $S$ using a low-rank factorization that leverages the correlation between the models. 

What makes UCB-E nonstandard in our finite-dataset setting is that the quantities used to construct the confidence bounds are both adaptive and non-stationary. The predicted scores are obtained by repeatedly applying a low-rank factorization to the partially observed score matrix, and hence depend on the entire history of previous evaluations. At the same time, the observed scores are also not i.i.d as they are sampled without replacement.

\subsection*{Related Work}

There is a rich line of research on the use of doubly robust estimators in the context of MAB~\cite{robins1995semiparametric,tsiatis2006semiparametric,dudik2011doubly,wang2017optimal,hadad2021confidence,kato2021adaptive,bibaut2021post,zhan2021off,dimakopoulou2021online,kim2019doubly,kim2021doubly}. Here, we discuss the works that are most closely related to ours. In \cite{ji2025multi}, the authors show how surrogate rewards can improve the sample efficiency of MAB by using PPI. However, they assume that the rewards' distribution of each arm is Gaussian and stationary. In contrast, we assume that the rewards population is finite (i.e., the entries of $S$), and changes with each step, as we sample scores from $S$ without replacement. 

Another closely related work to ours is  \cite{ao2026best}, which studies the problem of best-arm identification while using LLM judges and limited human audits. In contrast to our setting, this work assumes a stream of examples and scores that follows a stationary data-generating distribution, whereas we assume access to a finite number of examples, i.e., the complete dataset. Since we sample example-score pairs without replacement, the score distribution changes over time.

Finally, the FAQ method  \cite{wu2026efficient} considers the problem of efficiently evaluating the accuracy of a single candidate model using low-rank factorization. While the unbiased mean accuracy estimator used in our work follows the structure of the one from \cite{wu2026efficient}, we differ with respect to three key aspects. First, the FAQ algorithm samples examples with replacement, while we focus on sampling without replacement. Second, we construct CIs using finite-sample concentration inequalities, while the CIs used in FAQ are only asymptotically valid. Finally, we focus on the MAB setup, where the key challenge is to efficiently identify the best model, whereas  \cite{wu2026efficient} focuses on evaluating a single model. 

\section{Method}
\label{sec: the method}
\begin{algorithm}[t]
\caption{PULSE: Prediction-Powered Unbiased Low-Rank Sequential Evaluation}
\label{alg:ucb-e-pp}
\begin{algorithmic}[1]
\Require Test models $\mathcal{M}=\{f_1,\dots,f_m\}$, examples $\mathcal{X}=\{x_1,\dots,x_n\}$, scoring function $s(\cdot)\in\{0,1\}$;
evaluation budget $T$, batch size $B$, exploration parameter $a$;
low-rank model $\mathcal{M}$ with rank $r$, retrain period $\tau$;
init pulls per arm $T_0$;
side-information score matrix $S^{\text{train}}$
\Ensure Prediction $\hat{i}^*$ for the best model $i^*$.
\Statex \textbf{Stage 1: Initialization}
\State Fit $(U_{\text{train}}, V) \gets \mathcal{M}(S^{\text{train}}; r)$ using Step~1; freeze $V$ and discard $U_{\text{train}}$.
\For{$i = 1, \dots, m$}
    \State Sample $J_i \subset [n]$ with $|J_i|=B\times T_0$ and set $S_{i,j}\gets s(f_i(x_j))$ and $O_{i,j}\gets 1$  $\forall j\in J_i$.
\EndFor
\State $b \gets T - mBT_0$
\State Refit $U$ on the observed test entries with $V$ frozen; set $\hat{S}^0 \gets \sigma(U^TV)$.
\State For each $i$: initialize $\hat\mu_i$ and set $t,\hat\lambda^0_i \gets 0$.
\Statex \textbf{Stage~2: bandit loop}
\For{$t = 1, \dots, b/B$}
    \State $i \gets \argmax_{i:|O_i|<n} \hat{\mu}_i + \sqrt{a / |O_i|}$.
     \If{$t \bmod \tau = 0$} 
        \State Refit $U$ on the current observation mask $O$ with $V$ frozen;
        set $\hat{S}^t \gets \sigma(U^T V)$
    \EndIf
    \State compute $\hat{\lambda}^t_{i_t}$ using Step~3.
    \State sample $B_t \subset U^t_i$ with $|B_t| = B$
    \State $S_{i,j} \gets s(f_{i}(x_j))$ and $O_{i,j}\gets 1$ for every $j\in B_t$.
    \State Update $\hat{\mu}^t_i$
\EndFor
\State \Return $\hat{i}^* \gets \argmax_i \hat{\mu}_i$.
\end{algorithmic}
\end{algorithm}

In this section, we present our sequential approach for identifying the best-performing model while safely utilizing predicted scores to reduce the number of LLM calls. Our sequential method alternates between two key steps: (i) predicting the missing entries of the score matrix $S$ via low-rank factorization that utilizes the data evaluated thus far, and (ii) an adaptive MAB step that selects a model–example pair to evaluate, followed by updating the estimates of the models' mean accuracies $\mu_i$ and their corresponding CIs. For ease of exposition, we describe the method in the setting where a single example is sampled at each step. Later, we describe how to extend it to batch sampling.

\subsection*{Step 1: Low-Rank Factorization} 

Given a binary matrix $S \in \{0,1\}^{m \times n}$ and an observation
mask $O \in \{0,1\}^{m \times n}$ (with observed index set
$\Omega = \{(i,j) : O_{ij}=1\}$), we model the probability that model $i \in [m]$ answers correctly a question $j\in [n]$ by
\[
  \hat{S}_{i,j} =
  \sigma\left(u_i^{\top} v_j\right),
  \qquad \sigma(z)=\frac{1}{1+e^{-z}},
\]
where $u_i, v_j \in \mathbb{R}^r, \hspace{0.05in} 1\leq i \leq m, \hspace{0.05in} 1\leq j \leq n$ are learned low-dimensional latent representations of model $i$ and question $j$, respectively. The inner product $u_i^\top v_j$ is used to predict the log-odds that model $i$ answers question $j$ correctly.

We learn $u_i$ and $v_j$ by optimizing the following objective:
\begin{equation} 
\label{eq:low-rank-cost}
  \mathcal{L}(U,V) =
  \frac{1}{|\Omega|}
  \sum_{(i,j)\in\Omega}
  \mathrm{BCE}\left(S_{ij}, \sigma(u_i^{\top} v_j)\right)
  +
  \frac{\lambda}{2(m+n)}
  \bigl(\|U\|_F^{2} + \|V\|_F^{2}\bigr),
\end{equation}
where, for a label $y\in\{0,1\}$ and predicted probability $p\in(0,1)$, $\mathrm{BCE}(y,p)=-y\log p-(1-y)\log(1-p).$ Equivalently, this is the negative log-likelihood of a Bernoulli model in which
$\mathbb{P}(S_{ij}=1 \mid u_i,v_j)=\sigma(u_i^\top v_j).$
Thus, the first term in \eqref{eq:low-rank-cost} fits the observed entries of the score matrix using a low-rank logistic model. The regularization term prevents overfitting to the observed entries and promotes generalization to the unobserved entries, which is crucial in our setting since only a fraction of the score matrix is observed.
The parameter $\lambda$ controls the trade-off between fitting the observed scores and regularizing the latent factors.

\subsection*{Step 2: Estimating mean accuracies and constructing CIs}
We first fix a model $i\in[m]$ and show how to construct, at each time step $t$, a conditionally unbiased estimator of $\mu_i$ using the data collected during the sequential procedure. We begin by introducing the required notation. Let $O_i^{t-1}\subset[n]$ denote the set of indices of examples evaluated by model $i$ by the end of step $t-1$. Let $U_i^t=[n]\setminus O_i^{t-1}$ denote the set of indices of examples that model $i$ has not evaluated before step $t$. We denote by $\hat{S}_{i,j}^t$ the predicted score of model $i$ on example $j$ at step $t$. The superscript $t$ emphasizes that the low-rank factorization is updated over time, so the predicted scores may also change from one step to the next. Finally, since the next example is sampled uniformly from $U_i^t$,  we write $\pi^t_i = {1}/{|U^t_i|}$ for the probability of selecting any particular index $j \in U^t_i$.

To state the estimator and its unbiasedness property precisely, we also need to specify what information is available before the new example is sampled. This is captured by the natural filtration. Formally, let $\{\mathcal{F}_t\}_{t \ge 0}$ denote the filtration, where $\mathcal{F}_t$ is the sigma-algebra generated by all observations collected up to step $t$. A random variable $X$ is said to be $\mathcal{F}_t$-measurable if it is completely determined by the data available up to step $t$.

Now suppose that, at step $t$, model $i$ is selected and an index $j_t\in U_i^t$ is sampled uniformly. We define the following estimator for $\mu_i$:
\begin{equation}   
 \hat{\theta}_{i}^t =  \frac{1}{n} \Big[\sum_{j\in O^{t-1}_i} S_{i,j} + \sum_{j \in U_i^{t}} \lambda_{i}^{t}\hat{S}_{i,j}^{t}  + \underbrace{\frac{S_{i, j_t} - \lambda_{i}^{t}\hat{S}_{i,j_t}^{t}}{\pi^{t}_i}}_{Z_{i}^t} \Big]. 
\end{equation}
The first term sums the scores alow-rankeady observed for model $i$ before step $t$. The second term uses the predicted scores for the entries not yet observed for model $i$, scaled by a weight $\lambda^t_i \in [0,1]$ that can be tuned using past data, as discussed below. The final term, denoted by $Z_{i}^t$, is an inverse-probability residual correction that removes the bias introduced by using potentially inaccurate predicted scores for the unobserved entries. As the next lemma shows, this correction makes $\hat{\theta}^t_i$ a conditionally unbiased estimator for $\mu_i$.
\begin{lemma}
\label{lem:unbiased_estimator}
Assume $\pi^t_i, \lambda^t_i$ and $\hat{S}^t_{ij}$ are $\mathcal{F}_{t-1}$ measurable, then for each $t \ge 1$, $\mathbb{E}[\hat{\theta}_{i}^t \mid \mathcal{F}_{t-1}] = \mu_i$
\end{lemma}
Importantly, our algorithm satisfies the condition of Lemma~\ref{lem:unbiased_estimator} by construction. Indeed, at the beginning of step $t$, both $\pi_i^t$ and $\hat{S}^t_{ij}$ are fully determined by the data available up to the end of step $t-1$. Hence, both quantities are $\mathcal{F}_{t-1}$-measurable. Similarly, as we discuss below, the weight $\lambda^t_i$ is chosen using only past data, before observing the new score at step $t$. Therefore, $\lambda^t_i$ is also $\mathcal{F}_{t-1}$-measurable. 

To better understand the role of $\lambda_i^t$, consider the ideal case in which the predictions are exact, i.e., $\hat S_{ij}^t=S_{ij}$ for all $j\in U_i^t$. Then, setting $\lambda^t_i = 1$ yields the perfect estimator $\hat{\theta}_{i}^t \equiv \mu_i$.\footnote{For $\lambda^t_i=1$ and $S_{i,j} = \hat{S}^t_{i,j}$ we have that $Z^t_i=0$ and  $\hat{\theta}_{i}^t =  \frac{1}{n} \left[\sum_{j\in O^{t-1}_i} S_{i,j} + \sum_{j \in U_i^{t}} S_{i,j} \right] = \frac{1}{n}\sum_j S_{i,j}=\mu_i$.} Since $\mu_i$ is a finite-population mean, this perfect estimator has zero variance. At the other extreme, if the predictions are uninformative, the estimator remains conditionally unbiased by Lemma~\ref{lem:unbiased_estimator}, but using such poor predictions can increase variance. In this case, it is better to discard the predictions by setting $\lambda_i^t=0$. In Section~\ref{subsec:fine_tuning_lambda}, we formalize this intuition by showing that the variance of the estimator is controlled by the variance of the residual correction $Z_i^t$, which can be reduced by tuning $\lambda_i^t$.

The preceding construction gives a conditionally unbiased estimator $\hat\theta_i^t$ whenever model $i$ is selected at step $t$. To aggregate the information collected across multiple selections of the same model, we average these one-step estimators. Formally, let $A_i(t)$ be the set of steps in which the $i$-th model was selected. For example, if by time $t=10$ model $i=5$ was selected at steps $s=1$ and $s=7$, then $A_5(10) = \{1, 7\}$. We define the running estimator of $\mu_i$ at step $t$ as follows:
\begin{equation}    
\hat\mu^t_i := \frac{1}{|A_i(t)|}\sum_{s\in A_i(t)} \hat\theta^s_i.
\end{equation}

We next show how the estimator $\hat{\mu}_{i}^t$ can be used to construct a valid confidence interval for $\mu_i$. A key complication is that, even for a fixed model $i$, the sampled scores are not i.i.d. This is because examples are sampled without replacement: once model $i$ is evaluated on example $j$, this index is removed from the future sets $U_i^t$. Moreover, the predicted scores $\hat S_{ij}^t$ may change over time because the low-rank factorization is updated during the sequential procedure. As a result, the one-step estimators $\{\hat\theta_i^s:s\in A_i(t)\}$ are neither independent nor identically distributed. We therefore use martingale theory to derive a concentration bound on the estimation error.
\begin{theorem}
\label{thm:concentration_ineq}
        Fix model $i$ and step $t$, and assume that $|A_i(t)| = K > 0$. Define $V_i = \frac{1}{|A_i(t)|}\sum_{s \in A_i(t) } \text{Var}(Z_i^s | \mathcal{F}_{s-1})$. Assume that there exist a constant $v>0$ such that $V_i \le v$. Then, under the assumptions of Lemma~\ref{lem:unbiased_estimator},  for every $\epsilon > 0$,
     $$
   P\bigl(|\hat\mu_i^t - \mu_i| \ge \epsilon,\ |A_i(t)| = K\bigr)
\;\le\;
2\exp\!\left(-\frac{\epsilon^2}{2\bigl(\frac{v}{Kn^2} + \frac{2\epsilon}{3K}\bigr)}\right).
    $$
\end{theorem}
The theorem above clarifies how our CIs are \emph{prediction-powered}. Specifically, the bound depends on the conditional variance of the residual
correction $Z_i^t$, rather than directly on the conditional variance of the
observed score $S_{i,j_t}$. As we show next, the upper bound $v$ depends on both (i) how informative the predicted scores are about the true scores, and (ii) the choice of the weighting parameter $\lambda_i^t$.

\subsection*{Step 3: Tuning $\lambda^t_i$}
\label{subsec:fine_tuning_lambda}
We now study the effect of $\lambda_i^t$ on the estimator's variance and
derive a principled tuning approach. We begin by comparing our estimator to the
 observation-only estimator:
\begin{equation}
\label{eq:obs_only_estimator}
\hat{\mu}^{\text{obs}} = \frac{1}{|A_i(t)|} \sum_{s \in A_i(t)} S_{i,j_s}.
\end{equation}
Finite-sample confidence intervals for the observation-only estimator
can be constructed using the Bernstein--Serfling
inequality~\cite{bardenet2015concentration}. The width of the CI
for $\hat\mu_i^t$ is driven by the conditional variance appearing in
Theorem~\ref{thm:concentration_ineq}, while the width of the CI
for $\hat\mu^{\text{obs}}$ is driven by the finite-population (unconditional)
variance $\operatorname{Var}(S_i)$; we therefore compare these two variances
directly. The following lemma shows that an appropriate choice of
$\lambda_i^t$ reduces the conditional variance relative to the
observation-only baseline.
\begin{lemma}
\label{optimal_lambda}
    
Fix a model \(i\). Suppose that step \(t\) is the \(K\)-th selection of model
\(i\), so that \(|A_i(t)|=K\). The conditional variance of the residual correction $Z_i^t$ is minimized by choosing
\begin{equation}
(\lambda_i^t)^*
=
\frac{
\operatorname{Cov}\!\left(
S_{i,j_t},\,
\hat S^t_{i,j_t}
\,\middle|\, \mathcal F_{t-1}
\right)}
{
\operatorname{Var}\!\left(
\hat S^t_{i,j_t}
\,\middle|\, \mathcal F_{t-1}
\right)
}.
\label{eq:alpha-star-var-classical}
\end{equation}
At this choice of $(\lambda_i^t)^*$,
$$
 \text{Var}(\hat\mu^t_i | \mathcal{F}_{t-1}) \le (1- \rho^2_t)\text{Var}(\hat{\mu}^{\text{obs}}),
$$
where $
\rho_t = \operatorname{Corr}\!\left(S_{i,j_t},\,\hat S^t_{i,j_t}\,\middle|\, \mathcal F_{t-1}\right)$ is the conditional correlation between the true and predicted scores over the currently unobserved examples.
\end{lemma}

Since both Theorem~\ref{thm:concentration_ineq} (applied to $\hat\mu_i^t$)
and the Bernstein--Serfling inequality (applied to $\hat\mu^{\text{obs}}$)
yield confidence radii driven by the variances bounded above,
Lemma~\ref{optimal_lambda} implies that PULSE's CI is no wider than the
observation-only CI.\footnote{Up to constants of the respective concentration
inequalities.} Moreover, the variance reduction improves as the conditional
correlation $\rho_t$ between the true and predicted scores increases. Hence,
more informative predictions yield tighter confidence intervals through a
smaller residual variance.

In practice, computing $(\lambda_i^t)^*$ is not possible before observing
the new score $S_{i,j_t}$. However, to preserve the conditional unbiasedness guarantee, $\lambda_i^t$ must be chosen using only data available in $\mathcal F_{t-1}$, which does not include $S_{i,j_t}$. We therefore use a predictable plug-in approximation. The
idea is to approximate the conditional variance
\[
\operatorname{Var}(Z_i^t\mid \mathcal F_{t-1})
=
\mathbb E\!\left[
\left(Z_i^t-\mathbb E[Z_i^t\mid \mathcal F_{t-1}]\right)^2
\,\middle|\, \mathcal F_{t-1}
\right]
\]
using quantities that are measurable with respect to $\mathcal F_{t-1}$.
Specifically, we estimate the conditional mean
$\mathbb E[Z_i^t\mid \mathcal F_{t-1}]$ by the average of the previously
observed residual corrections,
\[
\bar Z_i^{t-1}
=
\frac{1}{|A_i(t-1)|}\sum_{s\in A_i(t-1)} Z_i^s,
\]
and approximate the unknown future score $S_{i,j_t}$ by its current prediction
$\hat S_{i,j_t}^t$. This yields the following predictable choice of
$\lambda_i^t$:
\begin{equation*}
\hat{\lambda}_{i}^t
=
\operatorname{clip}\left(
1 -
\frac{
\left(\sum_{j\in U_i^t}\hat S_{ij}^t\right)\bar Z_i^{t-1}
}{
|U_i^t|\sum_{j\in U_i^t}(\hat S_{ij}^t)^2
},
0,1
\right).
\end{equation*}

\subsection*{Putting it all together}
We now assemble the components above into the full PULSE procedure, summarized
in Algorithm~\ref{alg:ucb-e-pp}. At each round, the algorithm selects a model
for evaluation, optionally refits the low-rank factorization model from
Step~1, updates $\hat\lambda_i^t$ using Step~3, samples a new example for the
selected model, and updates the running estimator $\hat\mu_i^t$ and its
confidence interval using Step~2.

The following theorem gives a finite-sample success guarantee for the algorithm,
despite adaptive sampling, sampling without replacement, and time-varying
predictions:
\begin{theorem}
\label{thm:bandits}
Let $i^*=\arg\max_{i\in[m]}\mu_i$ be the unique best model, and define
$\Delta_i=\mu_{i^*}-\mu_i$ for $i\neq i^*$ and $\Delta_{i^*} = \min_{i \ne i^{*}} \Delta_i$. Let $H_1=\sum_{i}\Delta_i^{-2}$ and $\hat i_T$ be the output of Algorithm~\ref{alg:ucb-e-pp} when executed with a budget $T$. Suppose the exploration parameter $a$ satisfies $0\le a\le \frac{25(T-m)}{36H_1}$, and assume the constant $v$ in
Theorem~\ref{thm:concentration_ineq} can be chosen uniformly over all models $i\in[m]$ and all pull counts $K\le T$. Then
\[
\mathbb P(\hat i_T=i^*)
\ge
1-\sum_{i=1}^m\sum_{K=1}^T 2\exp\left(
-\frac{\epsilon_K^2}{
2\left(
\frac{v}{K n^2}
+
\frac{2\epsilon_K}{3K}
\right)}
\right), \quad \epsilon_K=\frac{1}{5}\sqrt{\frac{a}{K}}.
\]
\end{theorem}

The theorem shows that the success probability is controlled by two main
quantities. The first is the exploration parameter $a$, which determines the
width of the confidence bounds used by Algorithm~\ref{alg:ucb-e-pp}. Larger
values of $a$ lead to larger confidence bounds, encourage more exploration, and can improve convergence. However, taking $a$ too large could violate the theorem condition, while taking $a$ too small could reduce the success probability.

The second quantity affecting the success probability is the variance
term $v$ appearing in the confidence bound. Hence, informative predictions,
together with an appropriate choice of $\lambda_i^t$, can reduce the number of
LLM evaluations needed to reach a target success  by decreasing $v$. Conversely, when
the predictions are uninformative, our tuning rule is designed to downweight
the predicted scores by choosing $\hat\lambda_i^t$ close to zero.

\paragraph{From single-example to batch evaluation.}
Since GPUs are most efficient when computations are batched, it is natural to
evaluate multiple examples for the selected model at each step rather than a
single example. This motivates a batch-mode extension of PULSE. The extension
only requires modifying the residual correction term. Specifically, suppose
that at step $t$ we sample a batch $B_t\subset U_i^t$ of new example indices
for model $i$, uniformly without replacement. If each index $j\in U_i^t$ has
inclusion probability $\pi_i^t$, then we replace the single-example correction
$Z_i^t$ by $\tilde Z_i^t
=
\sum_{j\in B_t}
({S_{ij}-\lambda_i^t\hat S_{ij}^t})/\pi_i^t.
$
For uniform batch sampling with $|B_t|=B$, we have $\pi_i^t=B/|U_i^t|$. All other components, $\hat\theta_i^t$, $\hat\mu_i^t$, and the tuning rule for $\lambda_i^t$, are unchanged.

\section{Experiments}
\label{sec:experiments}
\begin{figure*}[t]
\centering
\begin{subfigure}[t]{0.5\textwidth}
    \centering
\includegraphics[width=\textwidth]{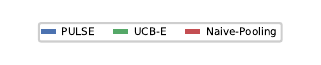}
\end{subfigure}
\\
\begin{subfigure}[t]{0.49\textwidth}
    \centering
    \includegraphics[width=\textwidth]{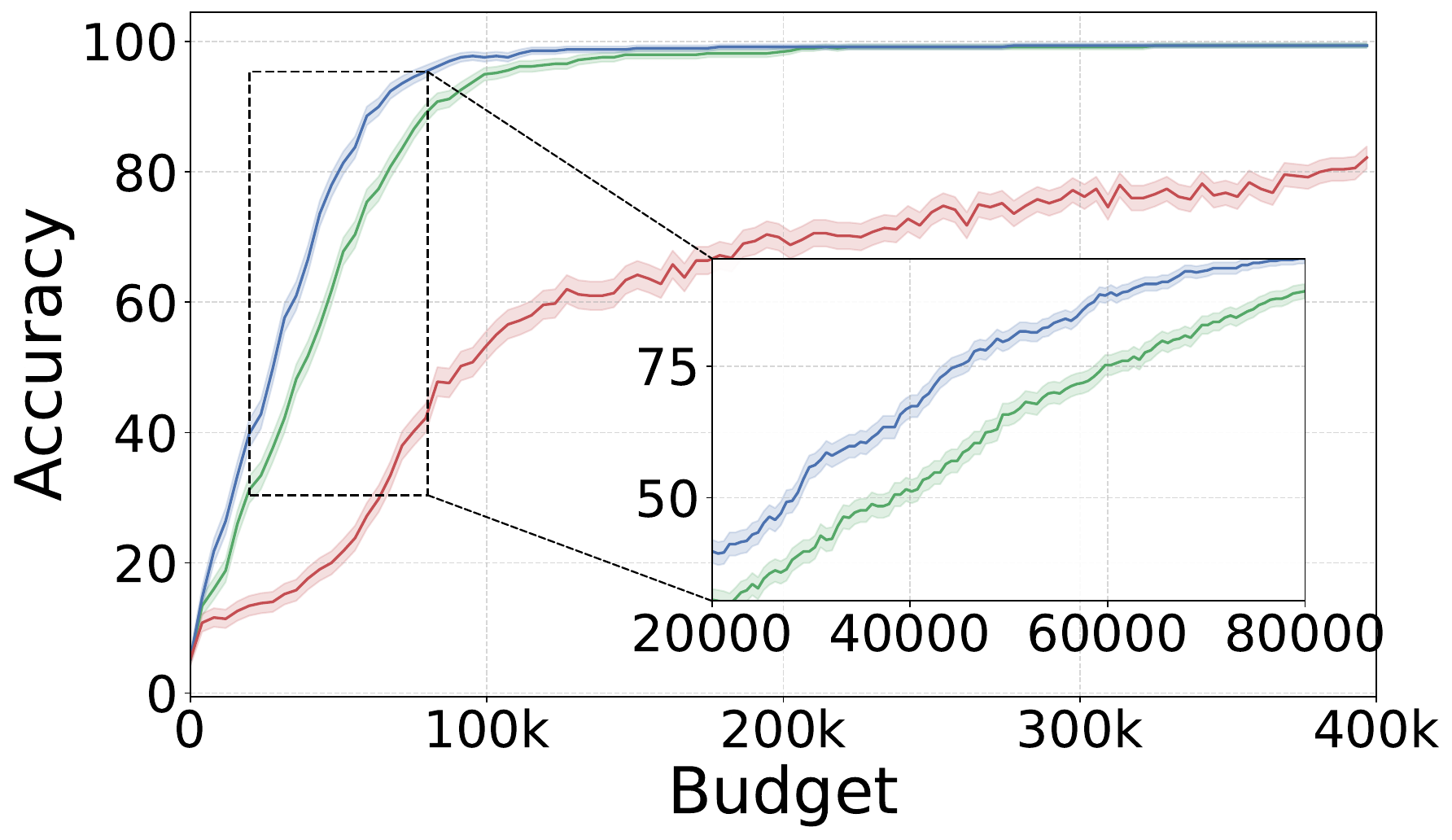}
    \subcaption{Bench 1($\Delta_{\text{top2}}= 0.02$)}
\end{subfigure}
\begin{subfigure}[t]{0.49\textwidth}
    \centering
    \includegraphics[width=\textwidth]{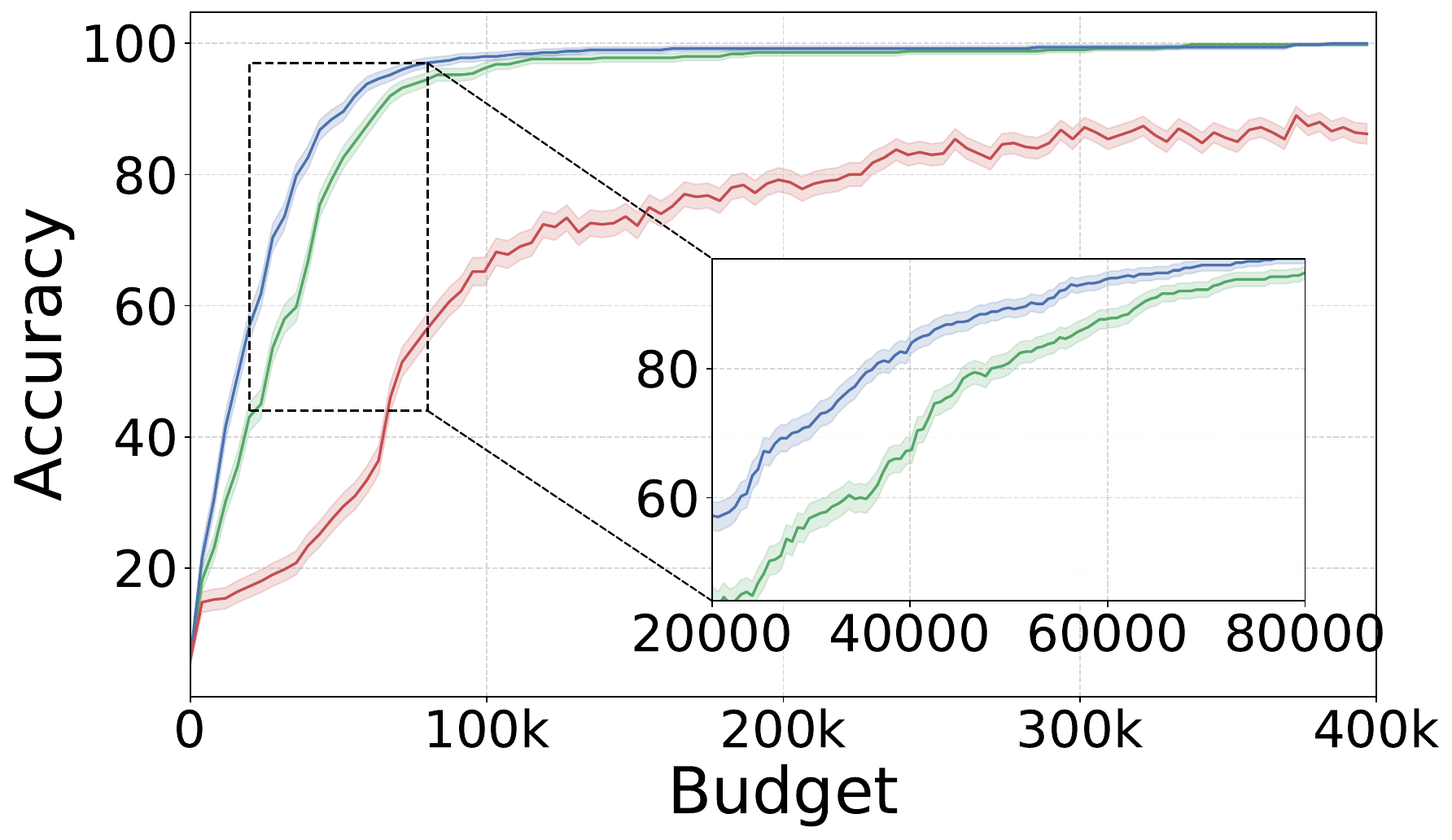}
    \subcaption{Bench 1($\Delta_{\text{top2}}= 0.03$)}
\end{subfigure}
\\
\begin{subfigure}[t]{0.49\textwidth}
    \centering
    \includegraphics[width=\textwidth]{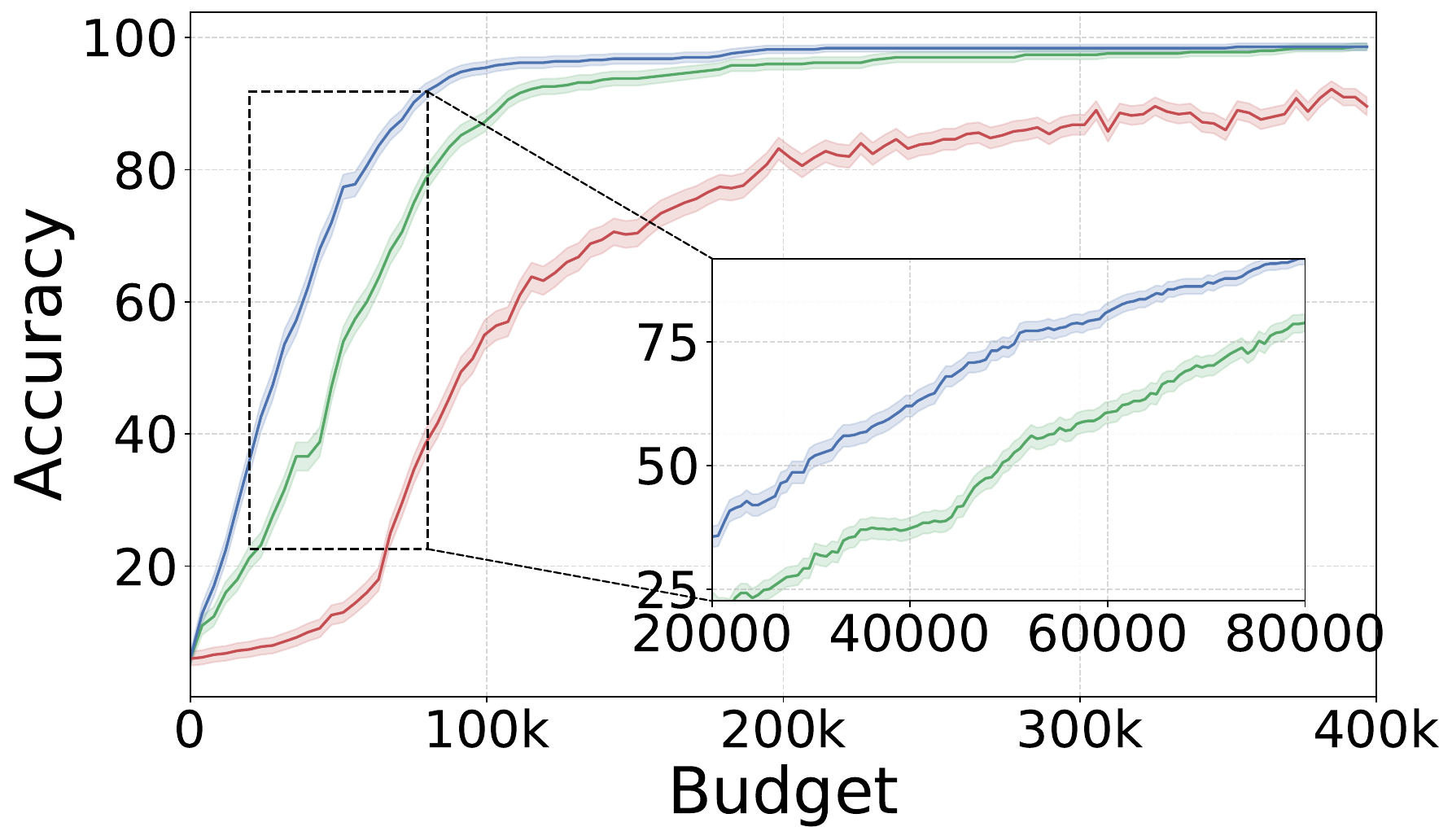}
    \subcaption{Bench 2($\Delta_{\text{top2}}= 0.02$)}
\end{subfigure}
\begin{subfigure}[t]{0.49\textwidth}
    \centering
    \includegraphics[width=\textwidth]{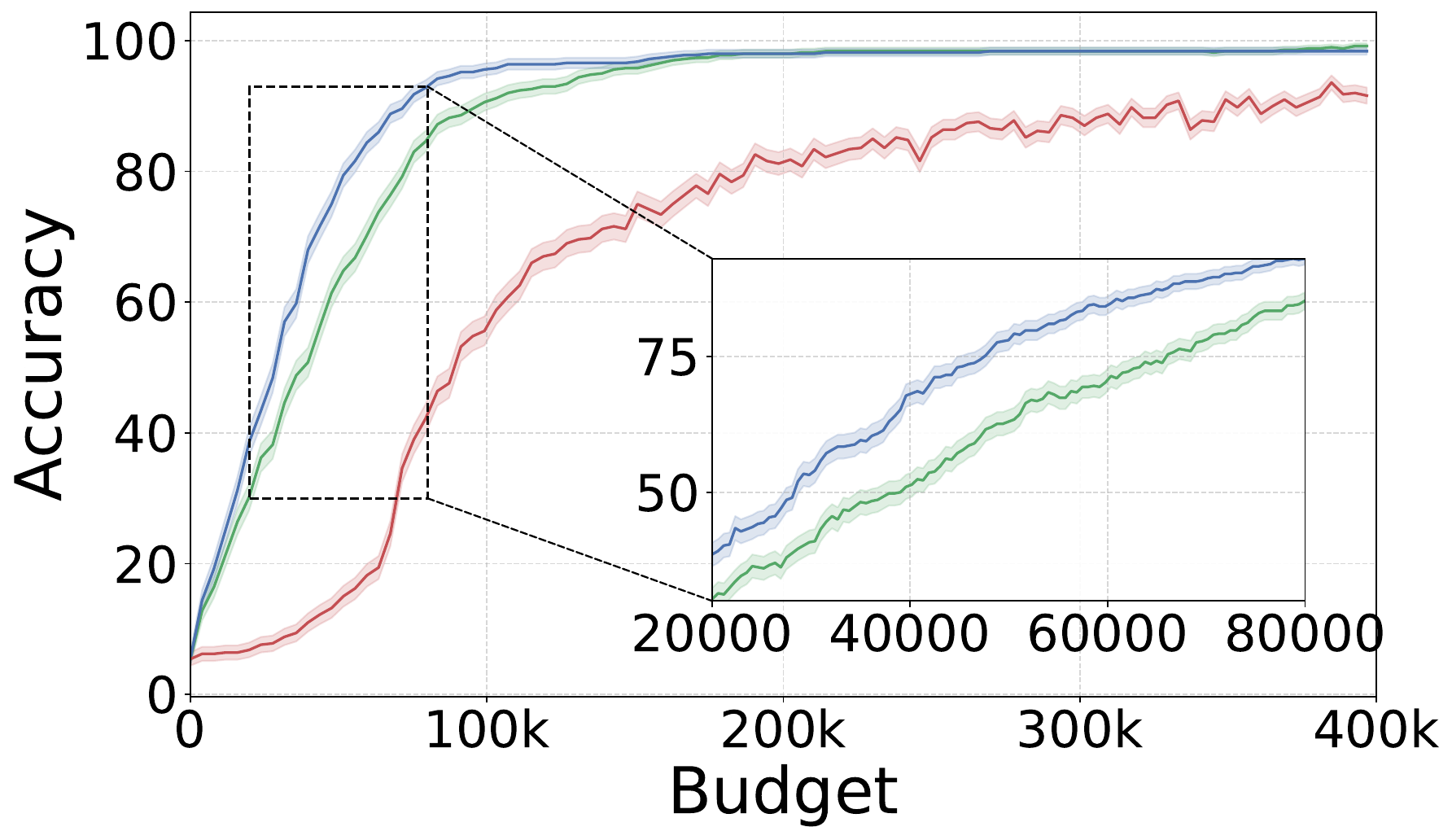}
    \subcaption{Bench 2($\Delta_{\text{top2}}= 0.03$)}
\end{subfigure}
\caption{Accuracy versus budget. Shaded areas around each curve depict standard error.}

\label{fig:acc_vs_budget}
\end{figure*}

We evaluate our algorithm on real-world benchmarks with the goal of assessing two key aspects: (i) the benefit of incorporating predicted scores, and (ii) the effect of correcting their bias when estimating model accuracies and constructing valid CIs.

\paragraph{Baselines.}
We compare PULSE against two baselines that share the UCB-E backbone but differ
in how they estimate each model's mean accuracy: (i)~\textbf{UCB-E}, the
standard UCB-E algorithm, which ignores the low-rank predictions and uses the
observation-only estimator~\eqref{eq:obs_only_estimator}; and
(ii)~\textbf{Naive-Pooling}, a UCB-E variant that does not account for
prediction bias and treats predicted scores as ground truth:
$
\hat{\mu}^{\text{pred}}_{i,t}
=
\frac{1}{n}\left(
\sum_{j \in O^{t-1}_i} S_{i,j}
+
\sum_{j \in U^{t}_i} \hat{S}^{t}_{i,j}
\right).
$

The first baseline isolates the value of using predictions: any gap between PULSE and \text{UCB-E} reflects variance reduction attributed to informative predictions. The second isolates the value of preserving statistical validity: any gap between PULSE and Naive-Pooling reflects the benefit of bias correction.

All methods use batch size \(64\) and total budget \(T=540{,}800\) LLM calls (\(T_0=1\) warmup batch per model plus \(400{,}000\) for the bandit loop). PULSE and Naive-Pooling refit the low-rank factorization every \(1{,}000\) steps. We run \(500\) independent repetitions and report, at each step, the fraction in which the method identifies the true best model. Full parameters in~\autoref{appendix:experiments}.

\paragraph{Performance metrics.} We consider the following metrics: (i) Accuracy@$t$: the probability of correctly identifying the best model at step $t, \hspace{0.05in} 1 \leq t \leq T$ (ii) Budget needed to obtain 95\% accuracy: the budget that each method requires to reach 95\% level of accuracy.

\paragraph{Datasets} We use the evaluation dataset released by~\cite{wu2026efficient}, which
provides binary correctness scores for $4.4K$ models across six benchmarks
with $21.5K$ total number of questions. Following~\cite{wu2026efficient}, we organize the
data into two benchmark datasets:
Bench 1 which contains~\emph{MMLU-Pro}~\cite{wang2024mmlu}, with 12K questions; and Bench 2 - a composite dataset of BBH~\cite{suzgun2022h}, GPQA~\cite{rein2024gpqa},
IFEval~\cite{zhou2023instruction}, MATH~\cite{li2023camel}, and
MuSR~\cite{ICLR2024_3f8c7eb8}, with 9.5K questions. For each dataset, we split the models chronologically according to the models release date: the older half is used as training data to warm-start the low-rank factorization model (as described next), and the remaining half forms the test pool over which best-model identification is performed. To assess our method under varying levels of difficulty, we consider the gap in accuracy between the top-2 models, denoted by $\Delta_{\text{top2}}$, and examine different levels of $\Delta_{\text{top2}} \in \{0.01, 0.02, 0.03\}$. For each target $\Delta_{\text{top2}}$, we construct a subset of $1000$ models as follows. We first include the top-performing model from the full score matrix. We then select as the second-best model the one whose accuracy gap with respect to the top model is closest to $\Delta_{\text{top2}}$. The remaining $998$ models are chosen as the next highest-performing models in descending order of mean accuracy. Smaller values of $\Delta_{\text{top2}}$ correspond to more challenging instances of the problem.

\paragraph{Low-rank model}
We use the train split of each dataset to fit an initial low-rank factorization
\((U_{\text{train}}, V_{\text{train}})\). At test time, we initialize the
question representations by setting \(V \leftarrow V_{\text{train}}\) and keep
\(V\) fixed throughout the run, as described in Stage~1 of
Algorithm~\ref{alg:ucb-e-pp}. The model representations \(U\), in contrast, are
reinitialized and updated online using the observed test-time scores via the
refitting step in the algorithm. Specifically, after each batch of evaluations,
we refit \(U\) while keeping \(V\) fixed. This design leverages the fact that question representations are shared between
the train and test splits and can therefore be estimated from historical data.
In contrast, the representations of the test models must be inferred from
test-time observations, since these models are evaluated only during the
sequential procedure. 

Hyperparameters of the factorization---the regularization strength $\lambda$,
rank $r$, learning rate, and number of training steps---are selected with the
train split via 5-fold cross-validation procedure. Concretely, we use the 1-SE
rule~\cite{friedman2010regularization} that favors more parsimonious
configurations among those within one standard error of the best
cross-validated loss.

\subsection{Results}

For each dataset and test subset, we plot Accuracy@$t$ as a function of $t$.
\autoref{fig:acc_vs_budget} depicts the results for $\Delta_{\text{top2}}=0.02,0.03$.
Comparing \textsc{PULSE} to UCB-E, which relies solely on observed scores, highlights
the benefit of incorporating predicted: \textsc{PULSE} consistently achieves higher accuracy across both datasets, for all values of $\Delta_{\text{top2}}$, and throughout the entire budget range. Importantly, this improvement in accuracy translates into concrete savings in LLM calls. Figure~\ref{fig:budget_to_95_bars} further summarizes the budget required to achieve $95\%$ accuracy. As shown, \textsc{PULSE} reduces the number of required LLM calls by up to $46\%$ compared to UCB-E. Crucially, in terms of runtime, This gain comes with a small computational overhead: on MMLU-Pro with \(T=540\)K,
\textsc{PULSE} takes about \(60\) seconds to run, which is negligible relative to the
cost of the LLM evaluations saved.

To isolate the effect of using our unbiased estimators and valid CIs, we next compare \textsc{PULSE} to the Naive-Pooling baseline.Recall that this baseline uses predicted scores directly, without debiasing the estimated means or ensuring CI validity.. As the results show, failing to account for prediction bias leads to substantial performance degradation. Indeed, Naive-Pooling consistently exhibits the worst performance across both datasets, for all values of $\Delta_{\text{top2}}$, and throughout the entire budget range. This finding confirms that naively incorporating predicted scores without bias correction can be more harmful than not using predictions at all. Additional results, including other metrics, values of \(\Delta_{\text{top2}}\), and full-dataset experiments, are provided in~\autoref{appendix:additional_results}.

\section{Conclusions and Future Work}
\label{sec:conclusion and future work}

We studied the problem of identifying the best-performing LLM under a limited evaluation budget and proposed \textsc{PULSE}, a sequential framework that combines adaptive sampling with low-rank score prediction while preserving statistical validity. By leveraging AIPW, our method constructs unbiased estimators and valid CIs in a non-i.i.d.\ setting. Empirically we demonstrated that \textsc{PULSE} significantly reduces the number of required LLM evaluations while maintaining high identification accuracy. 

Our approach has several limitations. First, its effectiveness depends on the quality of the low-rank factorization; improving the prediction model - for example, by incorporating richer contextual information—could further enhance performance. Second, our analysis focuses on binary scoring functions, and extending the framework to more general (non-binary) scores is an important direction, which we expect can be achieved with relatively minor modifications. Finally, \textsc{PULSE} currently samples examples uniformly at each step. Extending the method to non-uniform, potentially adaptive sampling is a promising direction. Intuitively, for a given model, some examples may be more informative than others for estimating its mean accuracy. Recent work~\cite{wu2026efficient} explores this idea in the single-model setting via adaptive question selection, and integrating such strategies into our framework could further improve efficiency by jointly optimizing over models and examples.

\section{Acknowledgments}
This research was supported the European Union (ERC, SafetyBounds, 101163414). Views and opinions expressed are however those of the authors only and do not necessarily reflect those of the European Union or the European Research Council Executive Agency. Neither the European Union nor the granting authority can be held responsible for them.
Y.R. thanks the Career Advancement Fellowship, Technion.
The authors also thank Matteo Sesia for his valuable feedback and insightful discussions.

\bibliographystyle{unsrtnat}
\bibliography{example_paper}

\appendix

\section{Proofs}
\label{appendix:proofs}

\subsection{Proof of Lemma~\ref{lem:unbiased_estimator}}

\begin{lemma}
Assume $\pi^k_i, \lambda^k_i$ and $\hat{S}^k_{ij}$ are $\mathcal{F}_{k-1}$ measurable, then for each $k \ge 1$, $\mathbb{E}[\hat{\theta}_{i}^k \mid \mathcal{F}_{k-1}] = \mu_i$
\end{lemma}
   
\begin{proof}
\begin{align*}
 \mathbb{E}[\hat{\theta}^k_{i}\mid \mathcal{F}_{k-1}] =& \mathbb{E}\left[\frac{1}{n} \left[\sum_{j\in O^{k-1}_i} S_{i,j} + \sum_{j \in U_i^{k}} \lambda_{i}^{k}\hat{S}_{i,j}^{k}  + \left(\frac{S_{i, j_k} - \lambda_{i}^{k}\hat{S}_{i,j_k}^{k}}{\pi^{k}_i}\right) \right] \Biggm| \mathcal{F}_{k-1}\right] \\ =& \frac{1}{n} \sum_{j\in O^{k-1}_i} S_{i,j} + \frac{1}{n}\sum_{j \in U^k_i} \lambda^k_{i}\hat{S}_{i,j}^k  +\frac{1}{n}\mathbb{E}\left[  \frac{S_{i,j_k} - \lambda^k_{i} \hat{S}_{i,j_k}^k}{\pi^k_i} \Biggm| \mathcal{F}_{k-1} \right] \\
 =& \frac{1}{n} \sum_{j\in O^{k-1}_i} S_{i,j} + \frac{1}{n}\sum_{j \in U^k_i} \lambda^k_{i}\hat{S}_{i,j}^k + \frac{1}{n}\sum_{j \in U^k_i} \pi^k_i \left( \frac{S_{i,j} - \lambda^k_{i} \hat{S}_{i,j}^k}{\pi^k_i} \right) \\
 &= \frac{1}{n} \sum_{j\in O^{k-1}_i} S_{i,j} + \frac{1}{n}\sum_{j \in U^k_i} S_{i,j} = \frac{1}{n} \sum_{j=1}^n S_{ij}  = \mu_i 
\end{align*}
   
\end{proof}

\subsection{Proof of Theorem~\ref{thm:concentration_ineq}}

\begin{theorem}
        Fix model $i$ and step $t$, and assume that $|A_i(t)| = K > 0$. Define $V_i = \frac{1}{|A_i(t)|}\sum_{s \in A_i(t) } \text{Var}(Z_i^s | \mathcal{F}_{s-1})$. Assume that there exist a constant $v>0$ such that $V_i \le v$. Then, under the assumptions of Lemma~\ref{lem:unbiased_estimator},  for every $\epsilon > 0$,
     $$
    P\bigl(|\hat\mu_i^t - \mu_i| \ge \epsilon,\ |A_i(t)| = K\bigr)
\;\le\;
2\exp\!\left(-\frac{\epsilon^2}{2\bigl(\frac{v}{Kn^2} + \frac{2\epsilon}{3K}\bigr)}\right).
    $$
\end{theorem}

\begin{proof}
For $s\in A_i(t)$, Define $D^s_i = \hat{\theta}^s_i - \mathbb{E}[\hat{\theta}_i^s | \mathcal{F}_{s-1}]$.
     Recall that:
     $$\hat{\theta}_{i}^s =  \frac{1}{n} \sum_{j\in O^{s-1}_i} S_{i,j} + \frac{1}{n}\sum_{j \in U^s_i} \lambda_{i}^s\hat{S}_{i,j}^s  + \frac{1}{n} \left(\frac{S_{i, j_s} - \lambda_{i}^s\hat{S}_{i,j_s}^s}{\pi^s_i}\right). $$

     Given $\mathcal{F}_{s-1}$, the only random part of $\hat{\theta}_{i}^s$ is $Z^s_i = \frac{S_{i, j_s} - \lambda_{i}^s\hat{S}_{i,j_s}^s}{\pi^s_i}$. Thus:

     $$
     \mathbb{E}[\hat{\theta}_{i}^s |\mathcal{F}_{s-1}] = \frac{1}{n} \sum_{j\in O^{s-1}_i} S_{i,j} + \frac{1}{n}\sum_{j \in U^s_i} \lambda_{i}^s\hat{S}_{i,j}^s  + \frac{1}{n}\mathbb{E}[Z^s_i | \mathcal{F}_{s-1}].
     $$

     Therefore:
     
     $$
     \hat{\theta}_{i}^s - \mathbb{E}[\hat{\theta}_i^s | \mathcal{F}_{s-1}] =  \frac{1}{n}(Z^s_i - \mathbb{E}[Z^s_i | \mathcal{F}_{s-1}]).
     $$

     Next, since $\lambda^s_i,\hat{S}^s_{i,j},S_{i,j} \in [0,1]$ and since $\pi^s_i = 1/|U^s_i|$ we have that $ -|U^s_i|\le Z^s_i \le |U^s_i|$. Thus $- 2|U^s_i| \le Z^s_i - \mathbb{E}[Z^s_i |\mathcal{F}_{s-1}] \le 2|U^s_i|$. Since $|U^s_i| \le n$, we have that:

     \begin{equation} \label{eq:bound on mu_hat}
        |\hat{\theta}_{i}^s - \mathbb{E}[\hat{\theta}_i^s | \mathcal{F}_{s-1}]| \le 2.
    \end{equation}

      Additionally, the following holds:
      
      $$
      \sum_{s \in A_i(t)} \mathbb{E}[(D^s_i)^2 \mid \mathcal{F}_{s-1}] = \sum_{s \in A_i(t)} \mathbb{E}[(\hat{\theta}_{i}^s - \mathbb{E}[\hat{\theta}_i^s | \mathcal{F}_{s-1}])^2 \mid \mathcal{F}_{s-1}] = \sum_{s \in A_i(t)} \text{Var}(\hat{\theta}_{i}^s \mid \mathcal{F}_{s-1}).
      $$

      As before, given $\mathcal{F}_{s-1}$, the only random part of $\hat{\theta}^s_i$ is $Z^s_i$. Thus:

      $$
       \sum_{s \in A_i(t)} \text{Var}(\hat{\theta}_{i}^s \mid \mathcal{F}_{s-1}) = \frac{1}{n^2}\sum_{s \in A_i(t)} \text{Var}(\hat{Z}_{i}^s \mid \mathcal{F}_{s-1}).
      $$
      
      Combining this with the assumption that $V_i<v$ we get:
      \begin{equation}\label{eq:bound_on_mu_hat_variance}
          \sum_{s \in A_i(t)} \text{Var}(\hat{\theta}_{i}^s \mid \mathcal{F}_{s-1}) \le \frac{vK}{n^2}
      \end{equation}

      Finally, It holds that:
      \begin{equation}\label{eq:self normalized}
        \mathbb{E}[D^s_i | \mathcal{F}_{s-1}] = 0.  
      \end{equation}

      Combining~\eqref{eq:bound on mu_hat},~\eqref{eq:bound_on_mu_hat_variance}, and~\eqref{eq:self normalized} we get that $(D^s_i)_{s \in \mathbb{N}}$ is a martingale difference sequence (M.D.S) and $S^K_i = \sum_{s \in A_i(t)} D^s_i$ is a martingale with bounded differences. 
      
      Equivalently, define $\tilde D^s_i = D^s_i \cdot \mathrm{1}\{i_s = i\}$ for $s = 1, \ldots, t$; since $\mathrm{1}\{i_s = i\}$ is $\mathcal{F}_{s-1}$-measurable by construction, $(\tilde D^s_i)_{s=1}^t$ is a martingale difference sequence w.r.t.\ $(\mathcal{F}_s)$ with $S^K_i = \sum_{s=1}^t \tilde D^s_i$, and the bound and conditional-variance estimates above carry over. With slight abuse of notation, we drop the tilde and write $D^s_i$ for $\tilde D^s_i$ in what follows.
      
      By Freedman's inequality~\cite{dzhaparidze2001bernstein, freedman1975tail}
applied to the martingale $(\tilde D_i^s)_{s\ge 1}$ at the deterministic
time $t$, with $|\tilde D_i^s|\le 2$ from~\eqref{eq:bound on mu_hat} and
predictable quadratic variation $\langle S^{(i)}\rangle_t = \sum_{s=1}^{t}
\mathbb{E}[(\tilde D_i^s)^2\mid \mathcal{F}_{s-1}]$, for every $\epsilon'>0$,
\[
P\bigl(|S_t^{(i)}|\ge \epsilon',\ \langle S^{(i)}\rangle_t \le K v/n^2\bigr)
\;\le\;
2\exp\!\left(-\frac{(\epsilon')^2}{2\bigl(K v/n^2 + 2\epsilon'/3\bigr)}\right).
\]
We now translate this into a bound on the event of interest. On
$\{|A_i(t)|=K\}$, the padding indicators yield $S_t^{(i)} = \sum_{s\in A_i(t)} D_i^s
= K(\hat\mu_i^t - \mu_i)$, and the assumption $V_i\le v$ gives
$\langle S^{(i)}\rangle_t = \sum_{s\in A_i(t)}\operatorname{Var}(\hat\theta_i^s
\mid\mathcal{F}_{s-1}) \le K v/n^2$ by~\eqref{eq:bound_on_mu_hat_variance}.
Hence
\[
\{|\hat\mu_i^t - \mu_i| \ge \epsilon,\ |A_i(t)| = K\}
\;\subseteq\;
\bigl\{|S_t^{(i)}|\ge K\epsilon,\ \langle S^{(i)}\rangle_t \le K v/n^2\bigr\}.
\]
Setting $\epsilon' = K\epsilon$ in Freedman's bound and combining with this
inclusion yields
\[
P\bigl(|\hat\mu_i^t - \mu_i|\ge\epsilon,\ |A_i(t)| = K\bigr)
\;\le\;
2\exp\!\left(-\frac{\epsilon^2}{2\bigl(v/(Kn^2) + 2\epsilon/(3K)\bigr)}\right),
\]
    which completes the proof.

\end{proof}

\subsection{Proof of Lemma~\ref{optimal_lambda}}

\begin{lemma}
    
Fix a model $i$ and a step $t$. Assume model $i$ is selected at step $t$ and  $|A_i(t)|=K$. The conditional variance of the residual correction $Z_i^t$ is minimized by choosing
\begin{equation}
(\lambda_i^t)^*
=
\frac{
\operatorname{Cov}\!\left(
S_{i,j_t},\,
\hat S^t_{i,j_t}
\,\middle|\, \mathcal F_{t-1}
\right)}
{
\operatorname{Var}\!\left(
\hat S^t_{i,j_t}
\,\middle|\, \mathcal F_{t-1}
\right)
}.
\end{equation}
At this choice of $(\lambda_i^t)^*$,

$$
 \text{Var}(\hat\mu^t_i | \mathcal{F}_{t-1}) \le (1- \rho^2_t)\text{Var}(\hat{\mu}^{\text{obs}})
$$
where $
\rho_t = \operatorname{Corr}\!\left(S_{i,j_t},\,\hat S^t_{i,j_t}\,\middle|\, \mathcal F_{t-1}\right).$ is the conditional correlation between the true and predicted scores over the currently unobserved examples.
\end{lemma}

\begin{proof}
    Recall that:
    $$ 
        Z_{i}^t = \frac{S_{i, j_t} - \lambda_i^t\hat{S}_{i,j_t}^t}{\pi^t_i}.
    $$

    Thus:

    $$
    \text{Var}(Z^t_i\mid \mathcal{F}_{t-1}) = \mathbb{E}[(Z^t_i - \mathbb{E}[Z^t_i | \mathcal{F}_{t-1}])^2 \mid \mathcal{F}_{t-1}].
    $$

    Given $\mathcal{F}_{t-1}$, $\lambda^t_i$, $\pi^t_i$ are constant, thus:
    $\mathbb{E}[Z^t_i | \mathcal{F}_{t-1}] = \frac{1}{\pi^t_i}(\mathbb{E}[S_{i, j_t} \mid \mathcal{F}_{t-1}] - \lambda^t_i \mathbb{E}[\hat{S}_{i,j_t}^t \mid \mathcal{F}_{t-1}])$.

    Therefore:
    \begin{align*}
    (Z^t_i - \mathbb{E}[Z^t_i | \mathcal{F}_{t-1}])^2 =& \left[\frac{1}{\pi^t_i}(S_{i, j_t} - \lambda_i^t\hat{S}_{i,j_t}^t) - \frac{1}{\pi^t_i}(\mathbb{E}[S_{i, j_t} \mid \mathcal{F}_{t-1}] - \lambda^t_i \mathbb{E}[\hat{S}_{i,j_t}^t \mid \mathcal{F}_{t-1}])\right]^2 = \\
    &\frac{1}{(\pi^t_i)^2}((S_{i, j_t} - \mathbb{E}[S_{i,j_t} \mid \mathcal{F}_{t-1}]) - \lambda_i^t(\hat{S}_{i,j_t}^t - \mathbb{E}[\hat{S}_{i,j_t}^t \mid \mathcal{F}_{t-1}]))^2
    \end{align*}

    Denote $\eta = \frac{1}{\pi^t_i}(S_{i, j_t} - \mathbb{E}[S_{i,j_t} \mid \mathcal{F}_{t-1}])$, and $\xi = \frac{1}{\pi^t_i}(\hat{S}_{i,j_t}^t - \mathbb{E}[\hat{S}_{i,j_t}^t \mid \mathcal{F}_{t-1}])$.

    We have that:
    $ (Z^t_i - \mathbb{E}[Z^t_i | \mathcal{F}_{t-1}])^2 = \eta^2 + (\lambda^t_i)^2\xi^2 - 2\lambda^t_i\eta\xi $.

    Thus:
    \begin{equation}
    \label{eq:var_decomposition}
    \text{Var}(Z^t_i\mid \mathcal{F}_{t-1}) = \mathbb{E}[\eta^2 \mid \mathcal{F}_{t-1}] + (\lambda^t_i)^2 \mathbb{E}[\xi^2 \mid \mathcal{F}_{t-1}] - 2\lambda^t_i \mathbb{E}[\eta\xi \mid \mathcal{F}_{t-1}].
    \end{equation}

    This is a quadratic equation in $\lambda^t_i$, which optimum is:
    \begin{equation}
    \label{eq:optiaml_lam_eta_xi}
        (\lambda^t_i)^* = \frac{\mathbb{E}[\eta\xi \mid \mathcal{F}_{t-1}]}{\mathbb{E}[\xi^2 \mid \mathcal{F}_{t-1}]} 
    \end{equation}

    Substituting $\eta$ and $\xi$ we get
    $$
     (\lambda^t_i)^* = \frac{\text{Cov}\left(\frac{S_{i,j_t}}{\pi^t_i}, \frac{\hat{S}^t_{i,j_t}}{\pi^t_i} \mid \mathcal{F}_{t-1}\right)}{\text{Var}\left(\frac{\hat{S}^t_{i,j_t}}{\pi^t_i} \mid \mathcal{F}_{t-1}\right)}.
    $$

    Given $\mathcal{F}_{t-1}$, $\pi^t_i$ is a constant, thus we can pull it out of the variance and covariance to get:

     $$
     (\lambda_i^t)^*
=
\frac{
\operatorname{Cov}\!\left(
S_{i,j_t},\,
\hat S^t_{i,j_t}
\,\middle|\, \mathcal F_{t-1}
\right)}
{
\operatorname{Var}\!\left(
\hat S^t_{i,j_t}
\,\middle|\, \mathcal F_{t-1}
\right)
}.
    $$

    Next, plug in~\eqref{eq:optiaml_lam_eta_xi} in~\eqref{eq:var_decomposition} to get:

    \begin{align*}
        \text{Var}(Z^t_i\mid \mathcal{F}_{t-1})\big|_{(\lambda^{t}_i)^*} =& \mathbb{E}[\eta^2 \mid \mathcal{F}_{t-1}] + \left(\frac{\mathbb{E}[\eta\xi \mid \mathcal{F}_{t-1}]}{\mathbb{E}[\xi^2 \mid \mathcal{F}_{t-1}]}\right)^2 \mathbb{E}[\xi^2 \mid \mathcal{F}_{t-1}] - 2\frac{\mathbb{E}[\eta\xi \mid \mathcal{F}_{t-1}]}{\mathbb{E}[\xi^2 \mid \mathcal{F}_{t-1}]} \mathbb{E}[\eta\xi \mid \mathcal{F}_{t-1}] = \\
        & \mathbb{E}[\eta^2 \mid \mathcal{F}_{t-1}] + \frac{\left(\mathbb{E}[\eta\xi \mid \mathcal{F}_{t-1}]\right)^2}{\mathbb{E}[\xi^2 \mid \mathcal{F}_{t-1}]}  - 2\frac{\left(\mathbb{E}[\eta\xi \mid \mathcal{F}_{t-1}]\right)^2}{\mathbb{E}[\xi^2 \mid \mathcal{F}_{t-1}]} =\\
        & \mathbb{E}[\eta^2 \mid \mathcal{F}_{t-1}] - \frac{\left(\mathbb{E}[\eta\xi \mid \mathcal{F}_{t-1}]\right)^2}{\mathbb{E}[\xi^2 \mid \mathcal{F}_{t-1}]} = \mathbb{E}[\eta^2 \mid \mathcal{F}_{t-1}]\left(1 - \frac{\left(\mathbb{E}[\eta\xi \mid \mathcal{F}_{t-1}]\right)^2}{\mathbb{E}[\xi^2 \mid \mathcal{F}_{t-1}]\mathbb{E}[\eta^2 \mid \mathcal{F}_{t-1}]} \right) = \\
        & \text{Var}\left(\frac{S_{i,j_t}}{\pi^t_i }\mid \mathcal{F}_{t-1} \right)(1 - \rho_t^2) =|U^t_i|^2\text{Var}\left(S_{i,j_t}\mid \mathcal{F}_{t-1} \right)(1 - \rho_t^2).
    \end{align*}

    Where the last transition is due to $1/\pi^t_i = |U^t_i|$.

    Next, since
    $$
    \hat{\theta}_{i}^t =  \frac{1}{n} \sum_{j\in O^{t-1}_i} S_{i,j} + \frac{1}{n}\sum_{j \in U^t_i} \lambda_{i}^s\hat{S}_{i,j}^t  + \frac{1}{n} Z^t_i. 
    $$

    We have that:

    $$
    \text{Var}(\hat{\theta}_{i}^t \mid \mathcal{F}_{t-1}) = \frac{1}{n^2}\text{Var}(Z^t_i\mid \mathcal{F}_{t-1})
    $$

    Thus:
    $$
    \text{Var}(\hat{\theta}_{i}^t \mid \mathcal{F}_{t-1}) = \frac{|U^t_i|^2}{n^2}\text{Var}\left(S_{i,j_t}\mid \mathcal{F}_{t-1} \right)(1 - \rho_t^2)
    $$

    Finally, since $\hat\mu^t_i := \frac{1}{|A_i(t)|}\sum_{s\in A_i(t)} \hat\theta^s_i$ and $|A_i(t)|=K$ we have that:
    
    \begin{align}
        \text{Var}(\hat\mu^t_i | \mathcal{F}_{t-1}) &= \text{Var}\left( \frac{1}{|A_i(t)|}\sum_{s\in A_i(t)}\hat\theta^s_i \mid \mathcal{F}_{t-1}\right) \\
        &=\frac{1}{K^2}\text{Var}\left(\hat{\theta}_{i}^t \mid \mathcal{F}_{t-1} \right) \\
        &= \frac{1}{K^2}\frac{|U^t_i|^2}{n^2}\text{Var}\left(S_{i,j_t}\mid \mathcal{F}_{t-1} \right)(1 - \rho_t^2) \\
        \label{eq:lemma_2_var_estimator_bound}
        &\le \frac{1}{K^2}\frac{|U^t_i|^2}{n^2}\frac{n}{|U_i^t|}\cdot\operatorname{Var}(S_i)(1 - \rho_t^2).
    \end{align}

    Where the last inequality stems from Lemma~\ref{lem:cond_var_bound}. 
    
    The variance of the fully observed estimator is:
     $$
     \operatorname{Var}(\hat\mu^{\text{obs}}) = \frac{\operatorname{Var}(S_i)}{K}\cdot\frac{n-K}{n-1},
     $$
     where $\operatorname{Var}(S_i) = \frac{1}{n}\sum_j (S_{i,j}-\mu_i)^2$ is the finite-population variance~\cite{bardenet2015concentration}.

     Thus: 
     \begin{equation}
     \label{eq:lemma_2_var_score_var_observed_estimator}
     \operatorname{Var}(S_i)=\operatorname{Var}(\hat\mu^{\text{obs}})K\frac{n-1}{n-K}.
     \end{equation}

     Combining~\eqref{eq:lemma_2_var_estimator_bound} and~\eqref{eq:lemma_2_var_score_var_observed_estimator} yields:
     \begin{align*}
     \text{Var}(\hat\mu^t_i | \mathcal{F}_{t-1}) \le  \frac{1}{K^2}\frac{|U^t_i|^2}{n^2}\frac{n}{|U_i^t|}\cdot\operatorname{Var}(\hat\mu^{\text{obs}})K\frac{n-1}{n-K}(1 - \rho_t^2)
     \end{align*}
     
     Substitute $|U^t_i| = n-K +1$ we have that:

     \begin{align}
     \label{eq:lemma_2_var_estimator_var_observed_estimator}
     \text{Var}(\hat\mu^t_i | \mathcal{F}_{t-1}) \le 
     \frac{|U^t_i|}{|U^t_i| - 1}\frac{n-1}{nK}\operatorname{Var}(\hat\mu^{\text{obs}})(1 - \rho_t^2)
     \end{align}

     For $K=1$ we have that $|U^t_i| = n$ and:
     \begin{align} 
     \text{Var}(\hat\mu^t_i | \mathcal{F}_{t-1}) \le \frac{|U^t_i|}{|U^t_i| - 1}\frac{n-1}{nK}\operatorname{Var}(\hat\mu^{\text{obs}})(1 - \rho_t^2) = \\
    \label{eq:lemma_2_K_1}
     \frac{n}{n - 1}\frac{n-1}{n}\operatorname{Var}(\hat\mu^{\text{obs}})(1 - \rho_t^2) = \operatorname{Var}(\hat\mu^{\text{obs}})(1 - \rho_t^2).
     \end{align}

     For $K \le n-1$, we have that $|U^t_i| \ge 2$ which implies that $\frac{|U^t_i|}{|U^t_i| - 1} \le 2$. Thus:
     
     \begin{align*}
     \text{Var}(\hat\mu^t_i | \mathcal{F}_{t-1}) \le 
     \frac{2(n-1)}{nK}\operatorname{Var}(\hat\mu^{\text{obs}})(1 - \rho_t^2) \le \frac{2}{K}\operatorname{Var}(\hat\mu^{\text{obs}})(1 - \rho_t^2)
     \end{align*}

     As such, for $K \ge 2$, $\text{Var}(\hat\mu^t_i | \mathcal{F}_{t-1}) \le \operatorname{Var}(\hat\mu^{\text{obs}})(1 - \rho_t^2)$. Combining this with~\eqref{eq:lemma_2_K_1} yields that for any $K \in [1, n-1]$, $\text{Var}(\hat\mu^t_i | \mathcal{F}_{t-1}) \le \operatorname{Var}(\hat\mu^{\text{obs}})(1 - \rho_t^2)$

\end{proof}

\begin{lemma}
\label{lem:cond_var_bound}
For every $t \ge 1$ and $i \in [m]$, the conditional variance of $S_{i,j_t}$ given $\mathcal{F}_{t-1}$ is bounded by
\begin{equation}
\operatorname{Var}(S_{i,j_t} \mid \mathcal{F}_{t-1}) \;\le\; \frac{n}{|U_i^t|}\cdot\operatorname{Var}(S_i).
\end{equation}
\end{lemma}

\begin{proof}
Let $\bar S_{U_i^t} := \frac{1}{|U_i^t|}\sum_{j \in U_i^t} S_{i,j}$ denote the empirical mean of the scores of model $i$ over the currently unobserved examples. Since $j_t$ is sampled uniformly from $U_i^t$, we have
\[
\operatorname{Var}(S_{i,j_t} \mid \mathcal{F}_{t-1}) = \frac{1}{|U_i^t|}\sum_{j \in U_i^t}\bigl(S_{i,j} - \bar S_{U_i^t}\bigr)^2.
\]
Using the identity
\[
\sum_{j \in U_i^t}\bigl(S_{i,j} - \bar S_{U_i^t}\bigr)^2 = \sum_{j \in U_i^t}(S_{i,j} - \mu_i)^2 - |U_i^t|\bigl(\bar S_{U_i^t} - \mu_i\bigr)^2,
\]
and dropping the non-negative second term, we obtain
\[
\operatorname{Var}(S_{i,j_t} \mid \mathcal{F}_{t-1}) \le \frac{1}{|U_i^t|}\sum_{j \in U_i^t}(S_{i,j} - \mu_i)^2.
\]
Since each summand is non-negative, extending the sum to all $j \in [n]$ only increases it:
\[
\frac{1}{|U_i^t|}\sum_{j \in U_i^t}(S_{i,j} - \mu_i)^2 \le \frac{1}{|U_i^t|}\sum_{j=1}^{n}(S_{i,j} - \mu_i)^2 = \frac{n}{|U_i^t|}\cdot \operatorname{Var}(S_i),
\]
which completes the proof.
\end{proof}

\subsection{Proof of~\autoref{thm:bandits}}
\begin{theorem}
Let $i^*=\arg\max_{i\in[m]}\mu_i$ be the unique best model, and define
$\Delta_i=\mu_{i^*}-\mu_i$ for $i\neq i^*$ and $\Delta_{i^*} = \min_{i \ne i^{*}} \Delta_i$. Let $H_1=\sum_{i}\Delta_i^{-2}.$ Let $\hat i_T$ be the output of Algorithm~\ref{alg:ucb-e-pp} when executed with a budget $T$, and let $a$ satisfy $0\le a\le \frac{25(T-m)}{36H_1}$. Assume that the constant $v$ in
Theorem~\ref{thm:concentration_ineq} can be chosen uniformly over all models
$i\in[m]$ and all pull counts $K\le T$. Then
\[
\mathbb P(\hat i_T=i^*)
\ge
1-\sum_{i=1}^m\sum_{K=1}^T 2\exp\left(
-\frac{\epsilon_K^2}{
2\left(
\frac{v}{K n^2}
+
\frac{2\epsilon_K}{3K}
\right)}
\right), \quad \epsilon_K=\frac{1}{5}\sqrt{\frac{a}{K}}.
\]
\end{theorem}

\begin{proof}
For each $i$ and $K\ge 1$, let $\tau_i^{(K)} := \inf\{t : |A_i(t)| = K\}$
denote the wall-clock time at which arm $i$ is selected for the $K$-th time
(with $\tau_i^{(K)} = \infty$ if this never occurs). Since $\hat\mu_i^t$ is
piecewise constant in $t$ between consecutive pulls of arm $i$, the running
estimator at any $t\in[T]$ with $|A_i(t)| = K$ equals
$\hat\mu_i^{\tau_i^{(K)}}$.

For every $1 \le i \le m, 1 \le K \le T$ Consider the events:  
\[
\xi_{i,K} \;=\; \left\{\tau_i^{(K)} > T\right\} \,\cup\, \left\{|\hat\mu_i^{\tau_i^{(K)}} - \mu_i| < \tfrac{1}{5}\sqrt{a/K}\right\},
\qquad \xi = \bigcap_{i=1}^m \bigcap_{K=1}^T \xi_{i,K}.
\]

 From \cite{audibert2010best} we know that for every $0 \le a \le \frac{25 (T-m)}{36H_1}$ on $\xi$ we have that $\hat{i}_T = i^*$. To conclude the proof we show that
 $P(\xi) \ge 1 - \sum_{K=1}^T\sum_{i=1}^m\alpha_{i,K}$.
Fix $i$ and $K$. From~\autoref{thm:concentration_ineq} we have that:
    $$
        P\bigl(|\hat\mu_i^t - \mu_i| \ge \epsilon,\ |A_i(t)| = K\bigr)
\;\le\;
2\exp\!\left(-\frac{\epsilon^2}{2\bigl(\frac{v}{Kn^2} + \frac{2\epsilon}{3K}\bigr)}\right).
    $$

To control $\xi_{i,K}^c$, apply the argument of Theorem~\ref{thm:concentration_ineq}
to the stopped martingale $S_t^{(i,K)} := S_{t\wedge\tau_i^{(K)}}^{(i)}$.
Stopping preserves the martingale property, the increment bound
$|\Delta S_t^{(i,K)}|\le 2$, and the predictable quadratic variation bound
$\langle S^{(i,K)}\rangle_T \le K v/n^2$. On $\{\tau_i^{(K)} \le T\}$,
$S_T^{(i,K)} = K(\hat\mu_i^{\tau_i^{(K)}} - \mu_i)$, so Freedman applied at the
deterministic time $T$ gives
\[
P\bigl(\tau_i^{(K)}\le T,\ |\hat\mu_i^{\tau_i^{(K)}} - \mu_i| \ge \epsilon_K\bigr)
\;\le\; \alpha_{i,K}.
\]
Thus:

\[
 P\left( \xi_{i,K}\right)  \ge 1 -  \alpha_{i,K}
\]

where $\alpha_{i,K} = 2\exp\left(-\frac{\epsilon^2}{2(\frac{v}{Kn^2} + \frac{2}{K}\epsilon/3)}\right)$

Taking the union bound over $\xi_{i,K}$ we get $P(\xi) \ge 1 - \sum_{K=1}^T\sum_{i=1}^m\alpha_{i,K}$.

\end{proof}

\subsection{Derivation of \texorpdfstring{$\hat{\lambda}^k_i$}{lambda-hat-i-k}}

Our goal is to find the weighting parameter $\lambda^k_i$ that minimizes the approximate optimization objective:
\begin{equation}
    \mathbb{E}\left[ \left( Z_{i}^k - \bar Z_i^{<k} \right)^2 \Biggm| \mathcal{F}_{k-1} \right],
\end{equation}
where $\bar Z_i^{<k}$ is a predictable proxy for $\mathbb{E}[Z_i^k\mid\mathcal{F}_{k-1}]$ formed from past residuals. To simplify notation, let $n_k = |U^{k-1}_i|$ be the number of unobserved examples for model $i$ at step $k$, and define the historical residual mean as the average over the steps at which model $i$ was previously selected,
\[
\bar Z_i^{<k}
\;=\;
\frac{1}{|A_i(k-1)|}\sum_{s\in A_i(k-1)} Z_i^s,
\]
matching the convention used in Step~3 of the main text. Since $A_i(k-1)$ and the residuals $\{Z_i^s : s\in A_i(k-1)\}$ are
$\mathcal{F}_{k-1}$-measurable, $\bar Z_i^{<k}$ is treated as a constant
throughout the conditional expectation below.

By expanding the squared difference and adding/subtracting $\mathbb{E}[Z^k_i \mid \mathcal{F}_{k-1}]$, we can decompose the objective into variance and bias components:
\begin{align}
\mathbb{E}\left[ \left( Z_{i}^k - \bar{Z}_i^{<k} \right)^2 \Biggm| \mathcal{F}_{k-1} \right] &= \mathbb{E}\left[ \left( Z_i^k - \mathbb{E}[Z^k_i \mid \mathcal{F}_{k-1}] + \mathbb{E}[Z^k_i \mid \mathcal{F}_{k-1}] - \bar{Z}_i^{<k} \right)^2 \Biggm| \mathcal{F}_{k-1} \right] \nonumber \\
&= \text{Var}(Z^k_i \mid \mathcal{F}_{k-1}) \;+\; \left(\mathbb{E}[Z^k_i \mid \mathcal{F}_{k-1}] - \bar{Z}_i^{<k}\right)^2, \label{eq:mse_decomposition}
\end{align}
where the cross-term vanishes because $\mathbb{E}[(Z_i^k - \mathbb{E}[Z^k_i \mid \mathcal{F}_{k-1}]) \mid \mathcal{F}_{k-1}] = 0$ and the remaining terms are $\mathcal{F}_{k-1}$ measurable.

\paragraph{Variance and Expectation Proxies.}
Because $j_k$ is sampled uniformly from $U^{k-1}_i$, the inclusion probability is $\pi^k_i = 1/n_k$. The true conditional variance is:
\begin{align*}
\text{Var}(Z^k_i \mid \mathcal{F}_{k-1}) &= n_k^2 \, \text{Var}(S_{i,j_k} - \lambda^k_i \hat{S}^k_{i, j_k}\mid \mathcal{F}_{k-1}) \\
&= n_k \sum_{j \in U^{k-1}_i} \left( (S_{i,j}-\lambda^k_i \hat{S}^k_{i, j}) - \frac{1}{n_k}\sum_{j' \in U^{k-1}_i} (S_{i,j'}-\lambda^k_i \hat{S}^k_{i, j'}) \right)^2.
\end{align*}
Since the true scores $S_{i,j}$ for $j \in U^{k-1}_i$ are unknown, we substitute them with the predictable plug-in estimates $\hat{S}^k_{i,j}$. The term $(S_{i,j} - \lambda^k_i \hat{S}^k_{i,j})$ is therefore approximated by $(1 - \lambda^k_i)\hat{S}^k_{i,j}$. This yields the variance proxy:
\begin{equation}
\widetilde{\text{Var}}(Z^k_i \mid \mathcal{F}_{k-1}) \;=\; n_k (1 - \lambda^k_i)^2 \sum_{j \in U^{k-1}_i} (\hat{S}^k_{i,j} - \bar{S}^k_i)^2 \;=\; n_k (1-\lambda^k_i)^2 \mathrm{SS}^k_i,
\end{equation}
where $\bar{S}^k_i = \frac{1}{n_k} \sum_{j \in U^{k-1}_i} \hat{S}^k_{i,j}$ and $\text{SS}^k_i := \sum_{j \in U^{k-1}_i} \left(\hat S^k_{i,j} - \bar S^k_i\right)^2 $

Similarly, we approximate the predictable expectation. The true conditional expectation is $\mathbb{E}[Z^k_i \mid \mathcal{F}_{k-1}] = \sum_{j \in U^{k-1}_i} (S_{i,j} - \lambda^k_i \hat{S}^k_{i,j})$. Substituting $S_{i,j} \approx \hat{S}^k_{i,j}$, the predictable expectation proxy becomes:
\begin{equation}
\widetilde{\mathbb{E}}[Z^k_i \mid \mathcal{F}_{k-1}] \;=\; (1-\lambda^k_i) \sum_{j \in U^{k-1}_i} \hat{S}^k_{i,j} \;=\; (1-\lambda^k_i) F^k_i.
\end{equation}

\paragraph{Closed form for $\hat\lambda^k_i$.}
Plugging these proxies back into Eq.~\eqref{eq:mse_decomposition}, our approximate objective becomes:
\begin{align*}
&\widetilde{\text{Var}}(Z^k_i \mid \mathcal{F}_{k-1}) \;+\; \left(\widetilde{\mathbb{E}}[Z^k_i \mid \mathcal{F}_{k-1}] - \bar{Z}_i^{<k}\right)^2
= n_k (1-\lambda^k_i)^2 \mathrm{SS}^k_i \;+\; \left( (1-\lambda^k_i)F^k_i - \bar{Z}_i^{<k} \right)^2 \\
&= (1-\lambda^k_i)^2 \left( n_k \mathrm{SS}^k_i + (F^k_i)^2 \right) - 2(1-\lambda^k_i)F^k_i \bar{Z}_i^{<k} + (\bar{Z}_i^{<k})^2.
\end{align*}
By applying the algebraic identity from Lemma~\ref{lem:algebric identity}, we know that $n_k \mathrm{SS}^k_i + (F^k_i)^2 = n_k \Phi^k_i$, where $\Phi^k_i = \sum_{j \in U^{k-1}_i} (\hat{S}^k_{i,j})^2$. Thus, the objective simplifies to:
\begin{equation}
(1-\lambda^k_i)^2 n_k \Phi^k_i - 2(1-\lambda^k_i)F^k_i \bar{Z}_i^{<k} + (\bar{Z}_i^{<k})^2.
\end{equation}
This is a simple quadratic in $(1-\lambda^k_i)$. Taking the derivative with respect to $(1-\lambda^k_i)$ and setting it to zero yields:
\begin{equation}
2(1-\lambda^k_i)n_k \Phi^k_i - 2F^k_i \bar{Z}_i^{<k} = 0 \quad \implies \quad 1-\lambda^k_i = \frac{F^k_i \bar{Z}_i^{<k}}{n_k \Phi^k_i}.
\end{equation}
Solving for $\lambda^k_i$ and projecting the result back into the valid range $[0, 1]$ yields the final closed-form update rule used in the algorithm:
\begin{equation}
\hat{\lambda}_{i}^k = \text{clip}\left( 1 - \frac{F^k_i \bar{Z}_i^{<k}}{n_k \Phi^k_i}, \; 0, \; 1\right),
\end{equation}
which matches the formula presented in Step 3.

\begin{lemma}
\label{lem:algebric identity}
Let $U^{k-1}_i$ be a finite index set of size $n := |U^{k-1}_i|$, and let $\hat S^k_{i,j}$ be a real-valued sequence indexed by $j \in U^{k-1}_i$. Define
\[
F^k_i \;:=\; \sum_{j \in U^{k-1}_i} \hat S^k_{i,j}, \qquad \Phi^k_i \;:=\; \sum_{j \in U^{k-1}_i} (\hat S^k_{i,j})^2, \qquad \bar S^k_i \;:=\; \frac{F^k_i}{n}.
\]
Define the sum of squared deviations
\[
\mathrm{SS}^k_i \;:=\; \sum_{j \in U^{k-1}_i} \bigl(\hat S^k_{i,j} - \bar S^k_i\bigr)^2.
\]
Then
\begin{equation}
n \cdot \mathrm{SS}^k_i \;=\; n\, \Phi^k_i \;-\; (F^k_i)^2.
\label{eq:ss-identity}
\end{equation}
\end{lemma}
 
\begin{proof}
Expand the squared deviation:
\[
\bigl(\hat S^k_{i,j} - \bar S^k_i\bigr)^2 \;=\; (\hat S^k_{i,j})^2 \;-\; 2 \hat S^k_{i,j} \bar S^k_i \;+\; (\bar S^k_i)^2.
\]
Sum over $j \in U^{k-1}_i$:
\[
\mathrm{SS}^k_i \;=\; \sum_{j \in U^{k-1}_i} (\hat S^k_{i,j})^2 \;-\; 2 \bar S^k_i \sum_{j \in U^{k-1}_i} \hat S^k_{i,j} \;+\; \sum_{j \in U^{k-1}_i} (\bar S^k_i)^2.
\]
Identify each term:
\begin{itemize}
    \item $\sum_{j \in U^{k-1}_i} (\hat S^k_{i,j})^2 = \Phi^k_i$ by definition.
    \item $\sum_{j \in U^{k-1}_i} \hat S^k_{i,j} = F^k_i$, so the middle term is $-2 \bar S^k_i \, F^k_i$.
    \item $\sum_{j \in U^{k-1}_i} (\bar S^k_i)^2 = n (\bar S^k_i)^2$ since $\bar S^k_i$ is constant in $j$ and the sum has $n$ terms.
\end{itemize}
Therefore
\[
\mathrm{SS}^k_i \;=\; \Phi^k_i \;-\; 2 \bar S^k_i \, F^k_i \;+\; n (\bar S^k_i)^2.
\]
Substituting $\bar S^k_i = F^k_i / n$:
\[
\mathrm{SS}^k_i \;=\; \Phi^k_i \;-\; 2 \cdot \frac{F^k_i}{n} \cdot F^k_i \;+\; n \cdot \frac{(F^k_i)^2}{n^2} \;=\; \Phi^k_i \;-\; \frac{2 (F^k_i)^2}{n} \;+\; \frac{(F^k_i)^2}{n} \;=\; \Phi^k_i \;-\; \frac{(F^k_i)^2}{n}.
\]
Multiplying both sides by $n$ yields \eqref{eq:ss-identity}.
\end{proof}

\clearpage

\section{Additional Results}
\label{appendix:additional_results}

~\autoref{fig:effective_budget_mmlu} report the effective budget compared to UCB-E, defined as the number of LLM calls required by each method to achieve the level of accuracy achieved by UCB-E when called with Budget B. As can be seen, across all test settings, PULSE requires fewer evaluations than UCB-E to reach the same accuracy level.

\autoref{fig:accuracy_vs_budget_appendix} (b) and (d) show the Accuracy@$t$ curve of Bench 1 and 2, respectively (i.e., when evaluated on the entire datasets). \autoref{fig:accuracy_vs_budget_appendix} (a) and (c) show the Accuracy@$t$ curve for a subset of $1$K models selected from Bench 1 and 2, respectively, so that $\Delta_{\text{top2}}=0.01$.

\begin{figure*}
\centering
\begin{subfigure}[t]{0.6\textwidth}
    \centering
\includegraphics[width=\textwidth]{res_bbh_gpqa/legend.pdf}
\end{subfigure}
\\
\begin{subfigure}[t]{0.45\textwidth}
    \centering
    \includegraphics[width=\textwidth]{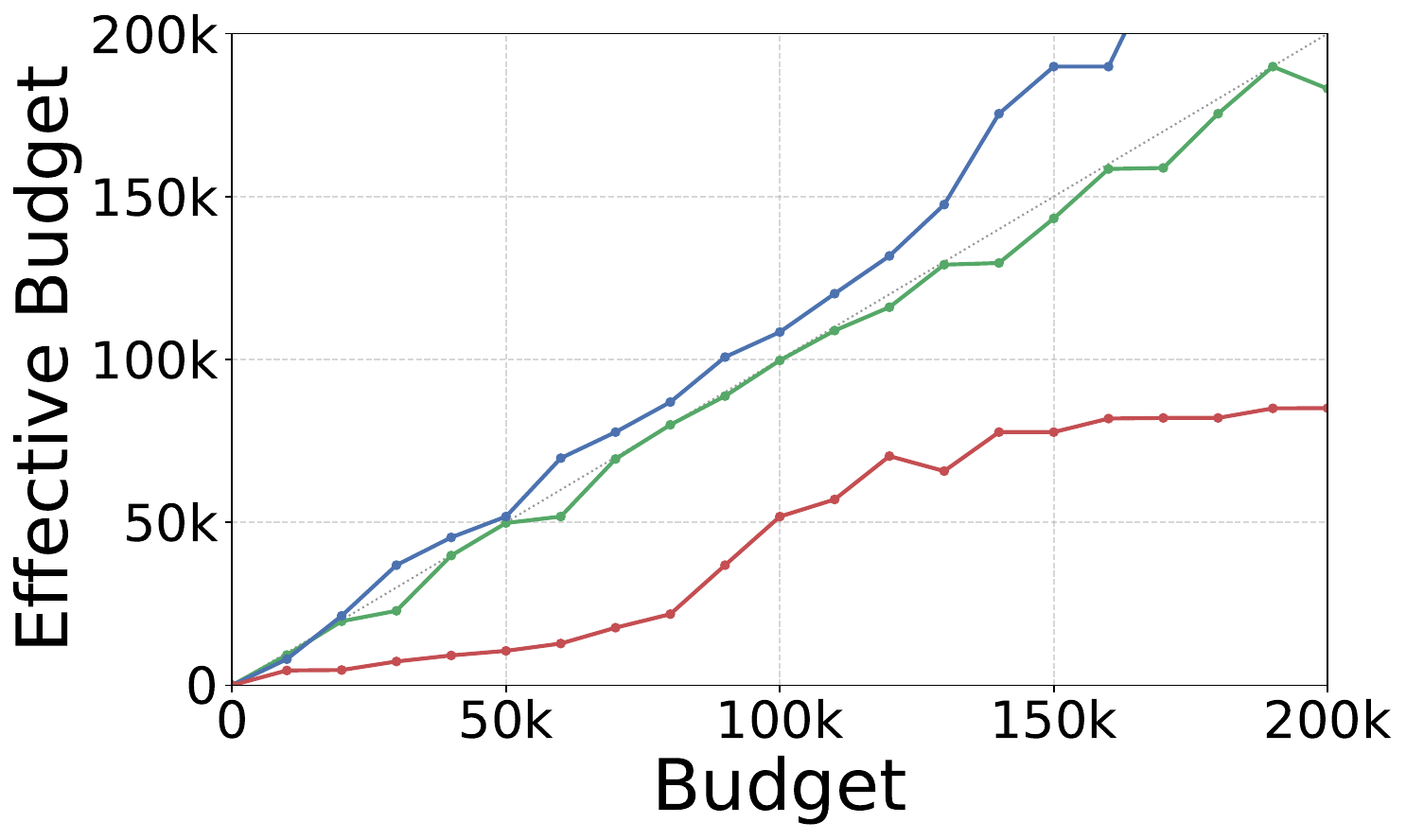}
    \subcaption{Bench 1(Full test set)}    
\end{subfigure}
\begin{subfigure}[t]{0.45\textwidth}
    \centering
    \includegraphics[width=\textwidth]{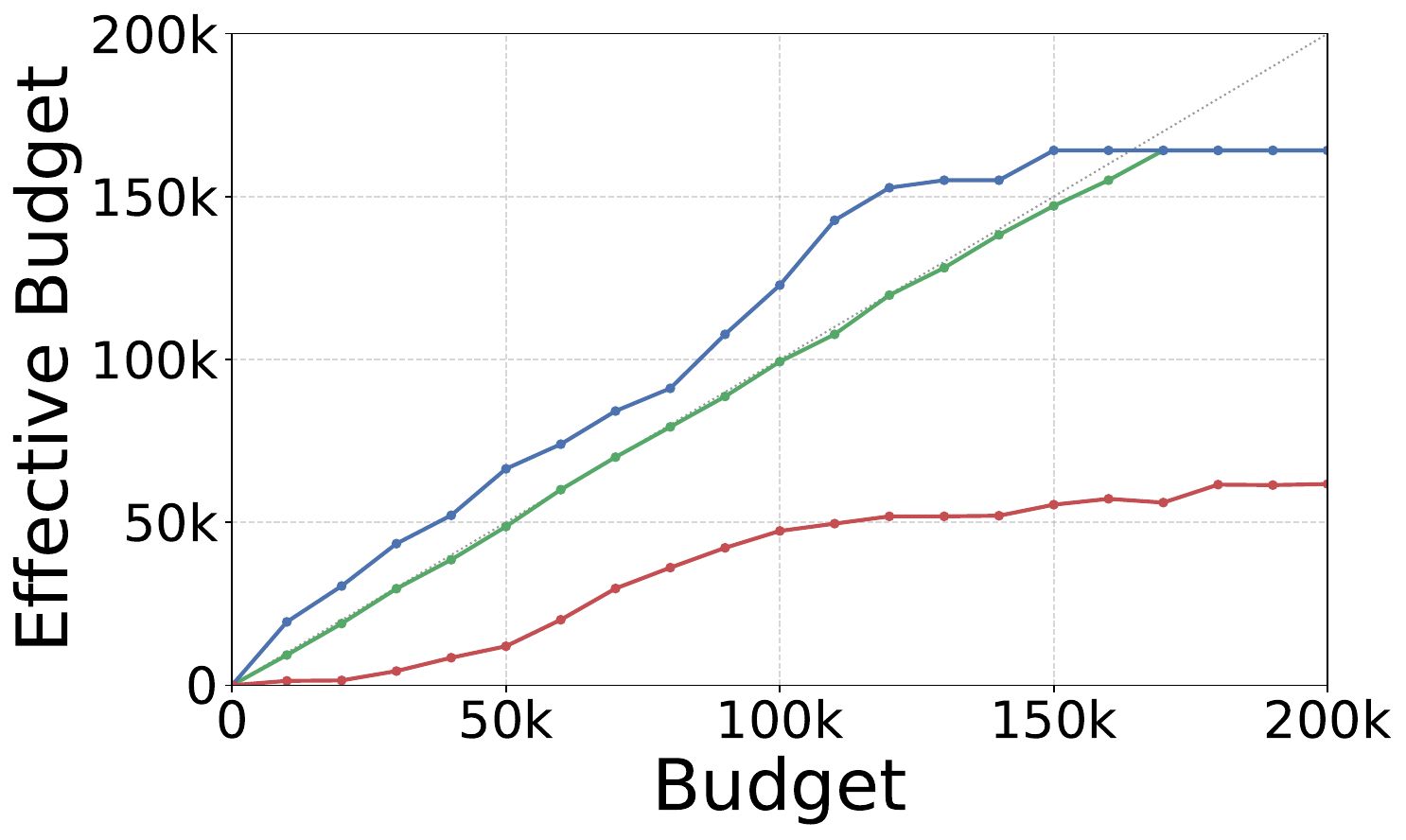}
    \subcaption{Bench 1 ($\Delta_{\text{top2}}=0.01$)}
\end{subfigure}
\\
\begin{subfigure}[t]{0.45\textwidth}
    \centering
    \includegraphics[width=\textwidth]{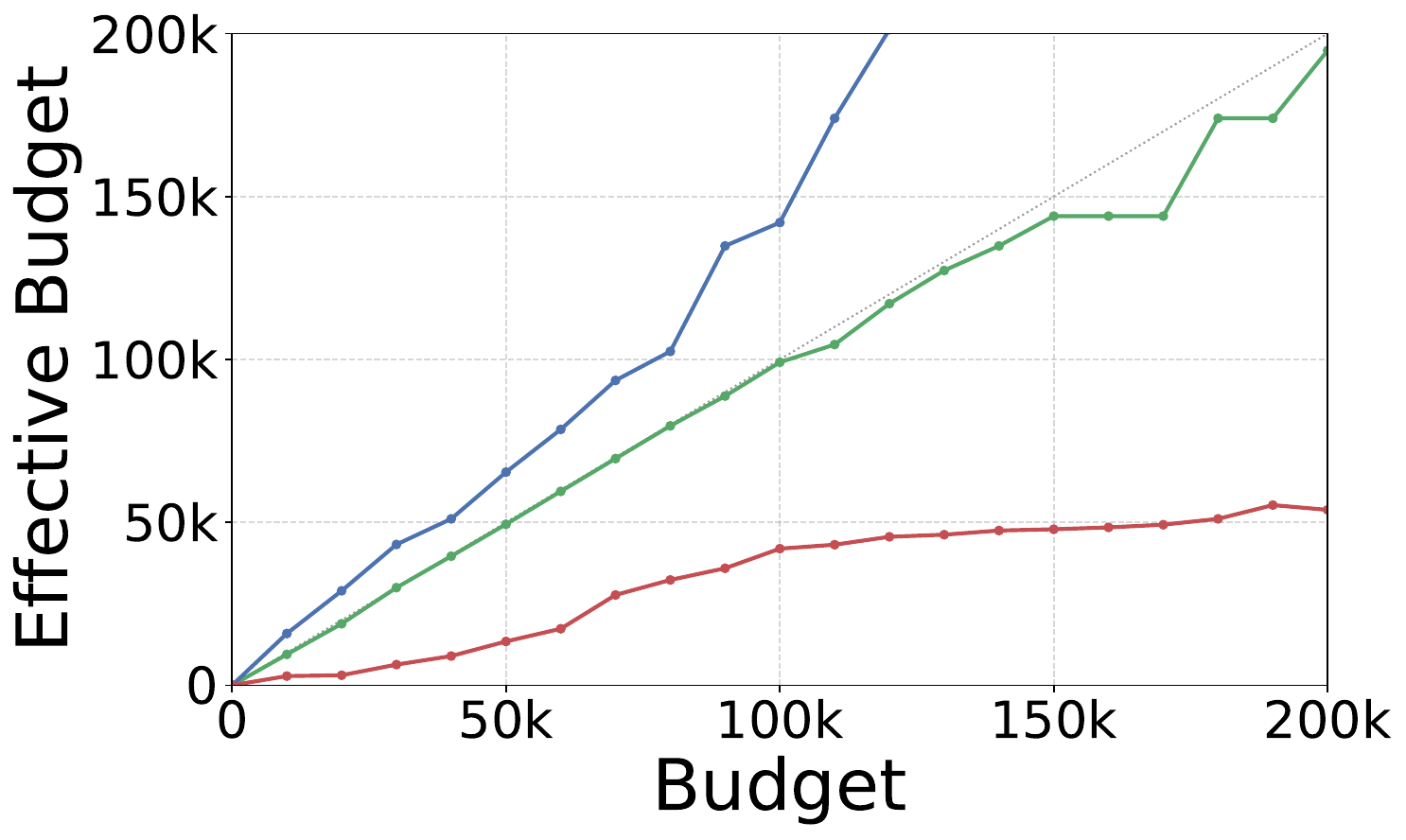}
    \subcaption{Bench 1 ($\Delta_{\text{top2}}=0.02$)}
\end{subfigure}
\begin{subfigure}[t]{0.45\textwidth}
    \centering
    \includegraphics[width=\textwidth]{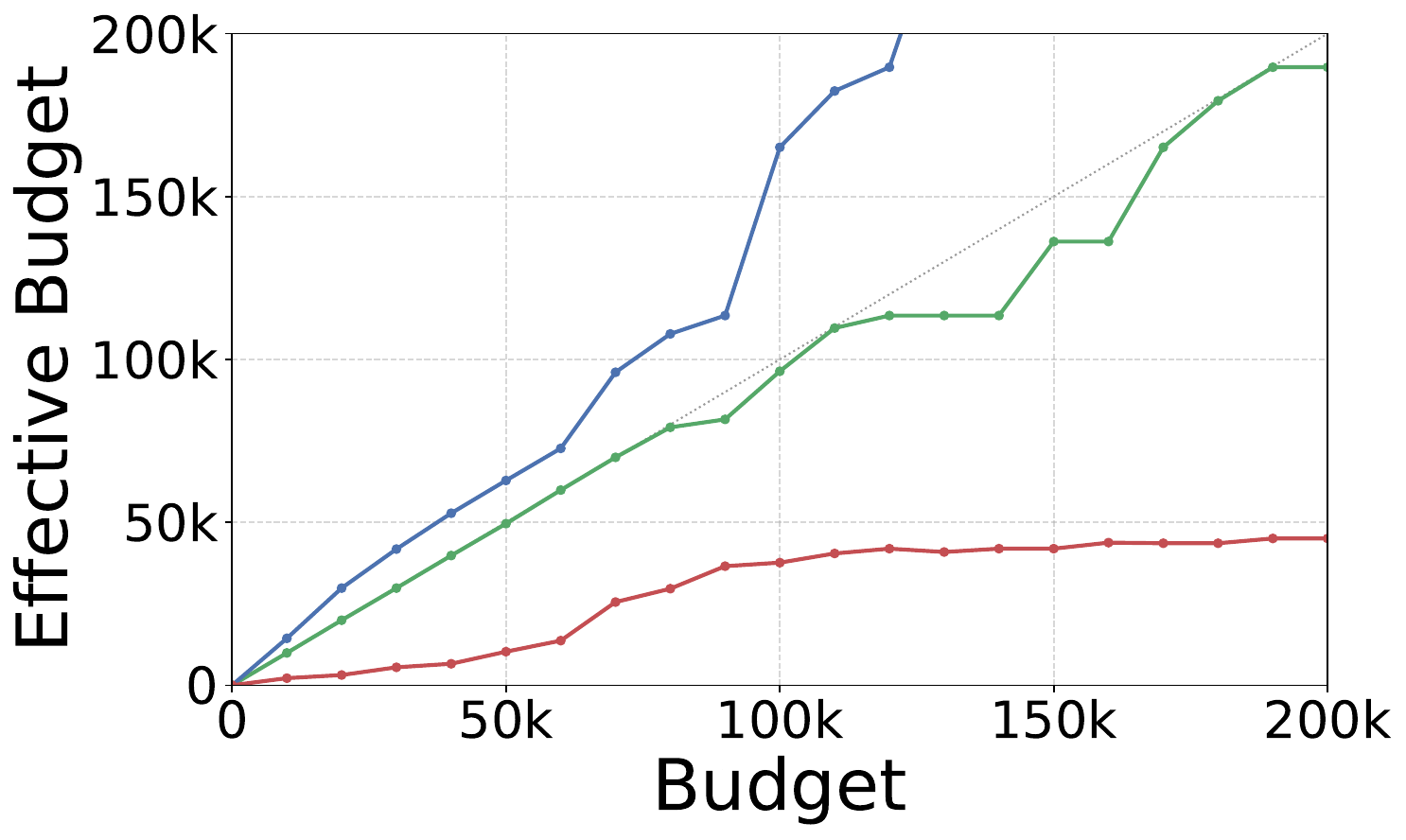}
    \subcaption{Bench 1 ($\Delta_{\text{top2}}=0.03$)}
\end{subfigure}
\begin{subfigure}[t]{0.45\textwidth}
    \centering
    \includegraphics[width=\textwidth]{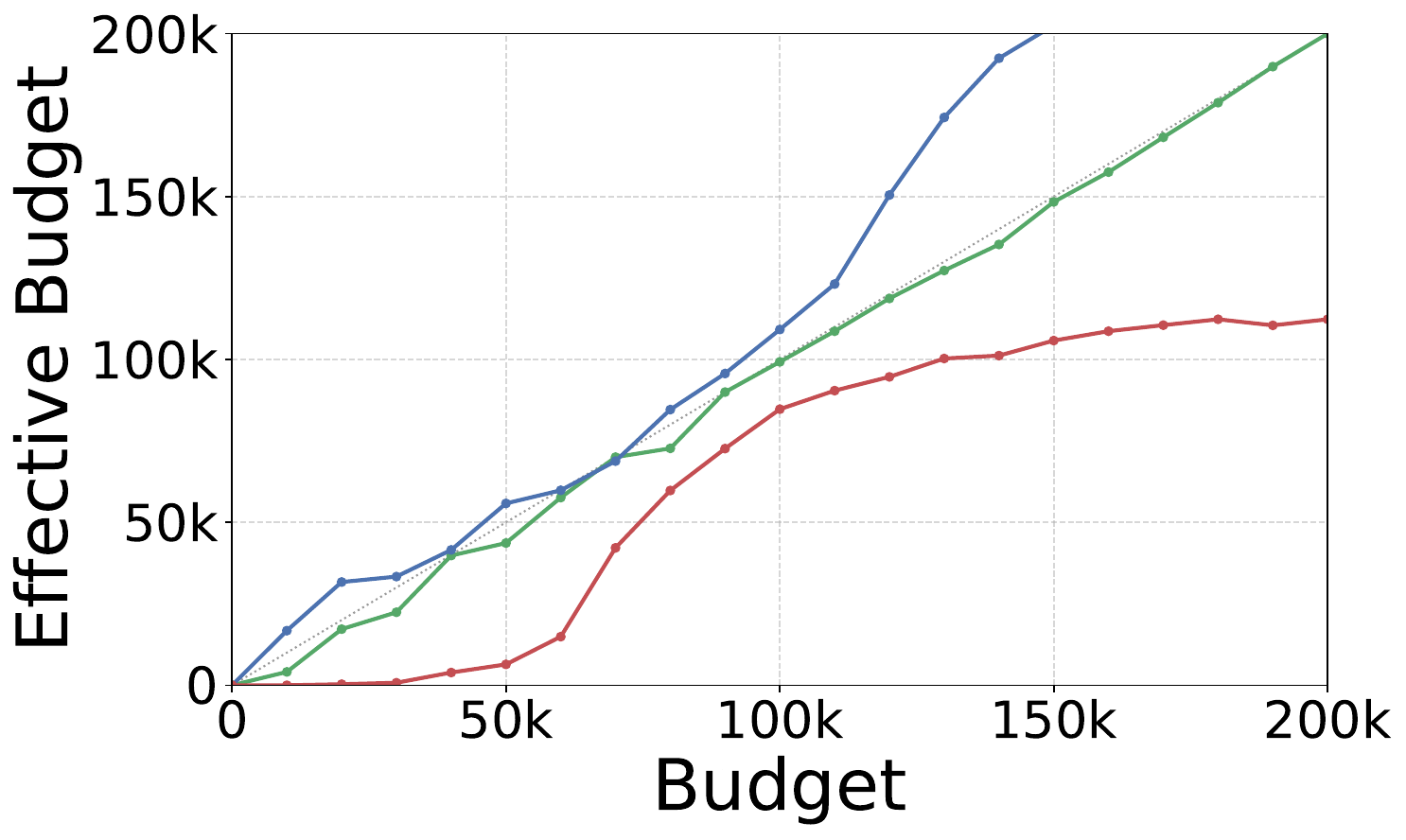}
    \subcaption{Bench 2(Full test set)}    
\end{subfigure}
\begin{subfigure}[t]{0.45\textwidth}
    \centering
    \includegraphics[width=\textwidth]{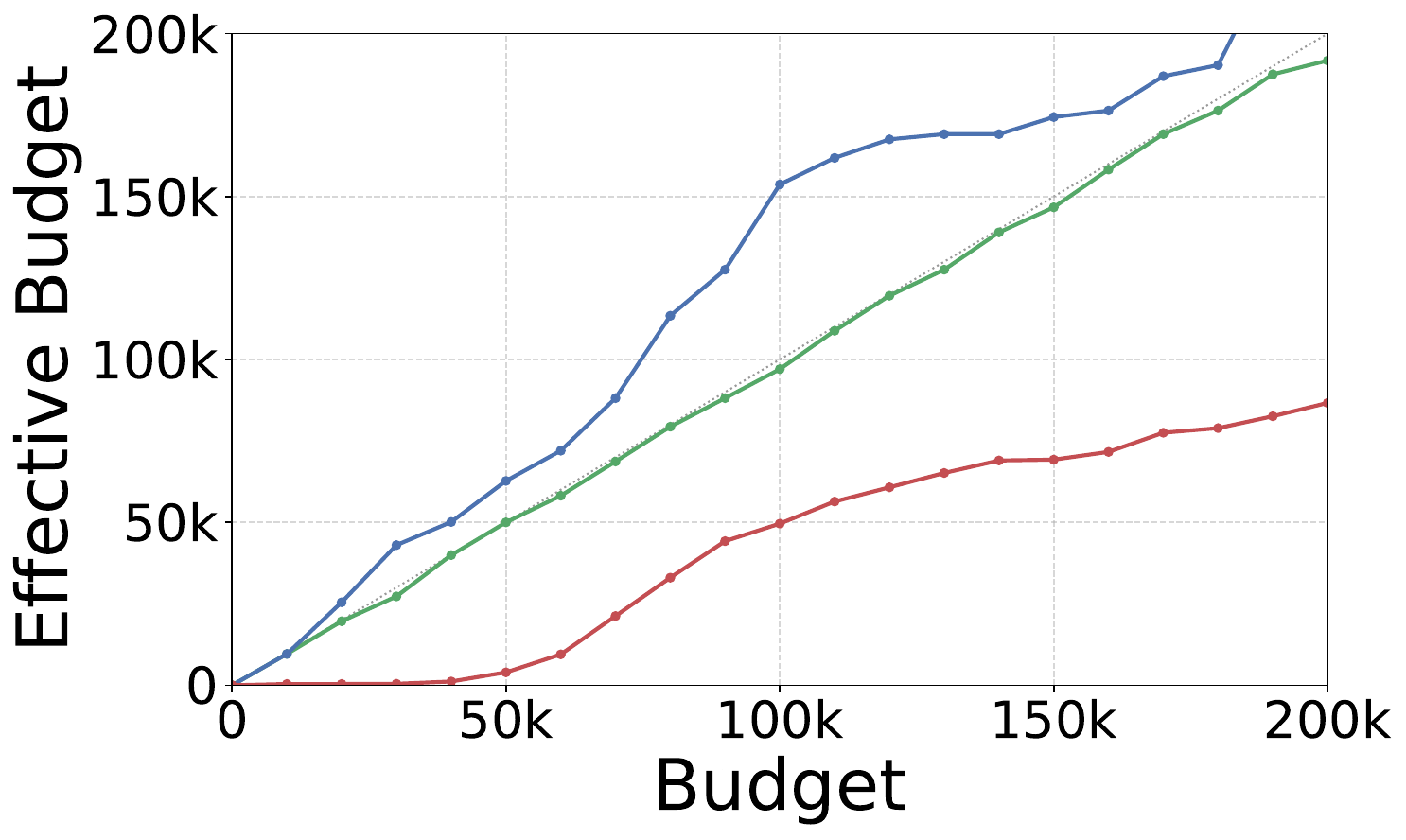}
    \subcaption{Bench 2 ($\Delta_{\text{top2}}=0.01$)}
\end{subfigure}
\\
\begin{subfigure}[t]{0.45\textwidth}
    \centering
    \includegraphics[width=\textwidth]{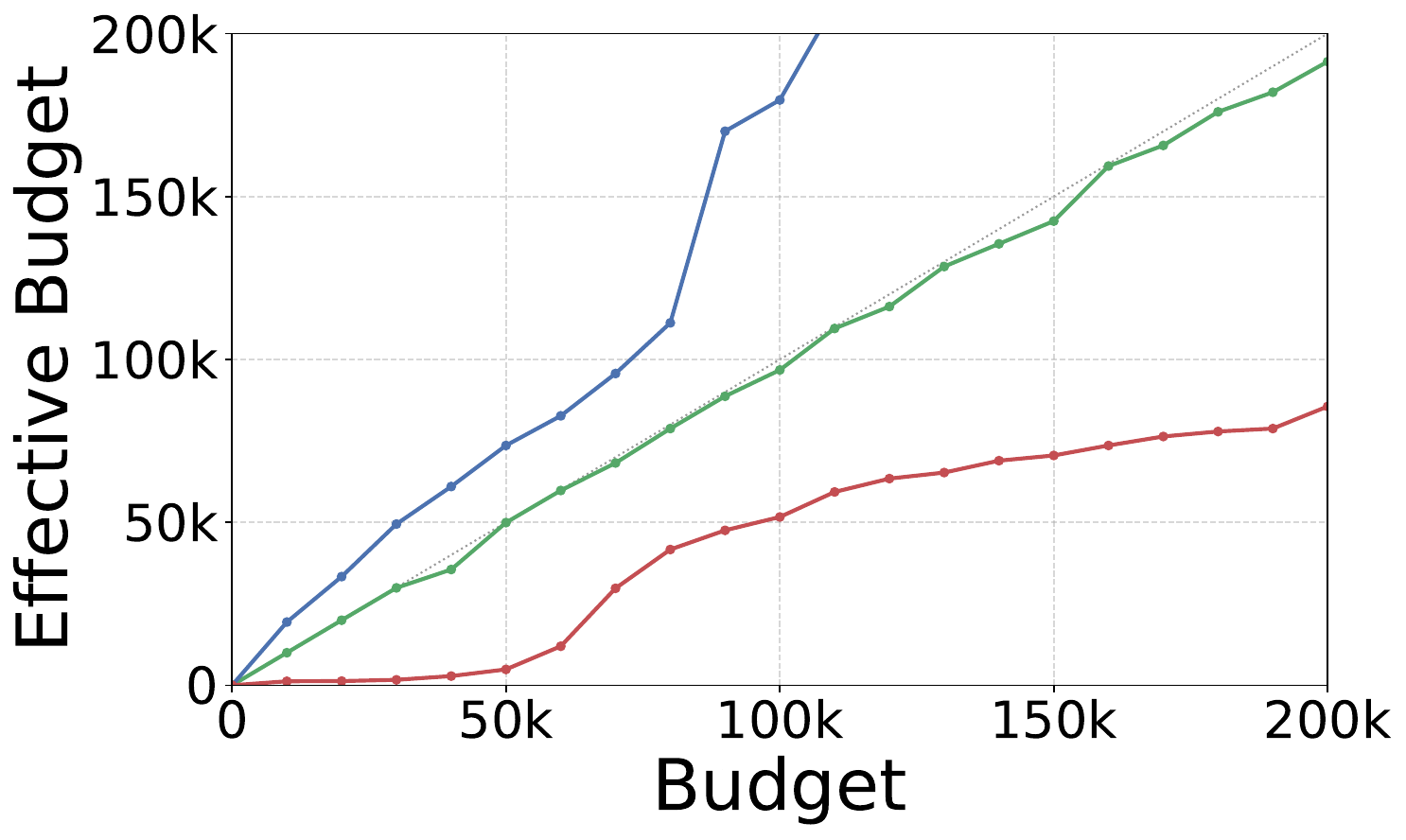}
    \subcaption{Bench 2 ($\Delta_{\text{top2}}=0.02$)}
\end{subfigure}
\begin{subfigure}[t]{0.45\textwidth}
    \centering
    \includegraphics[width=\textwidth]{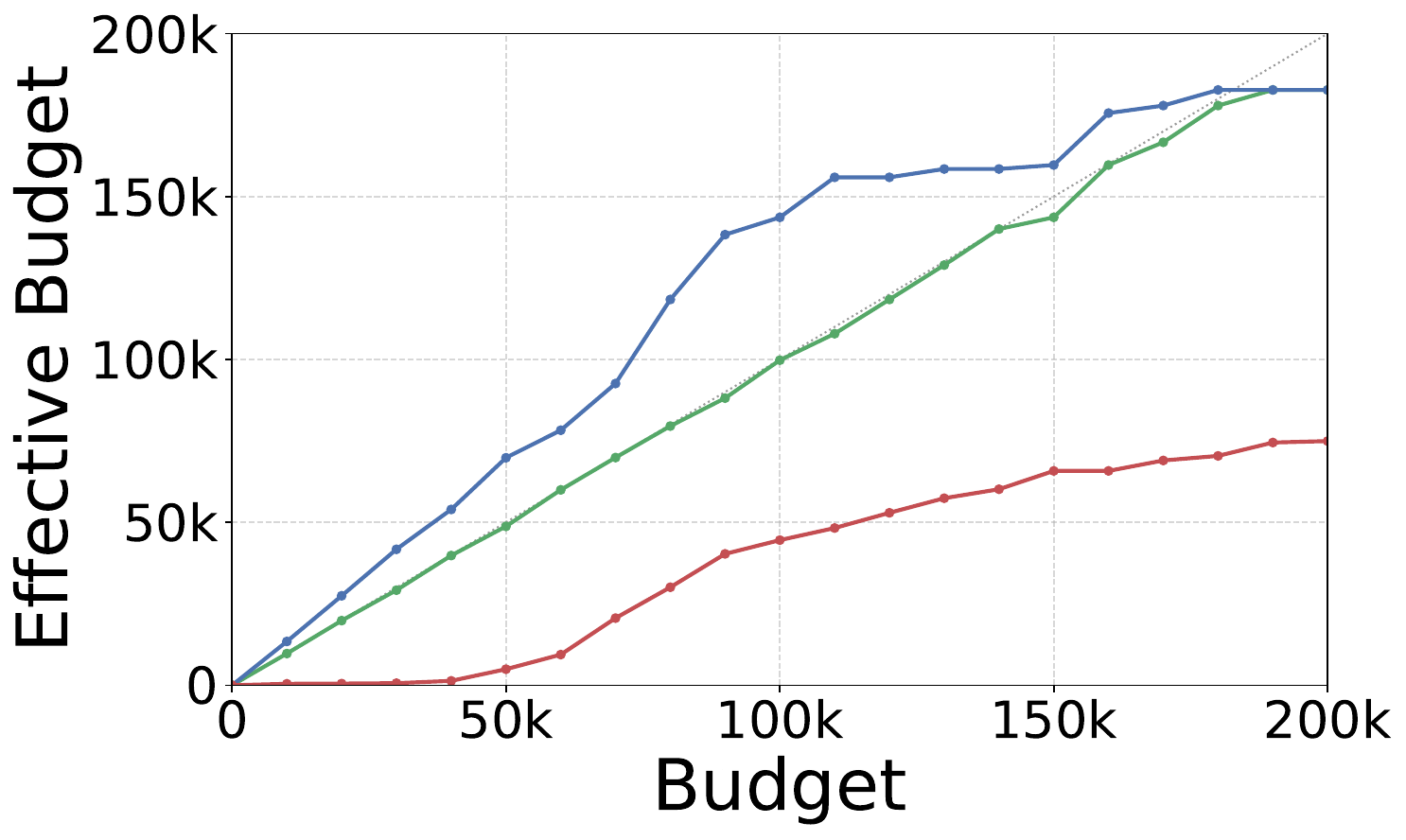}
    \subcaption{Bench 2 ($\Delta_{\text{top2}}=0.03$)}
\end{subfigure}
\caption{Effective Budget}
\label{fig:effective_budget_mmlu}
\end{figure*}

\begin{figure*}
\centering
\begin{subfigure}[t]{0.6\textwidth}
    \centering
\includegraphics[width=\textwidth]{res_bbh_gpqa/legend.pdf}
\end{subfigure}
\\
\begin{subfigure}[t]{0.45\textwidth}

    \centering
    \includegraphics[width=\textwidth]{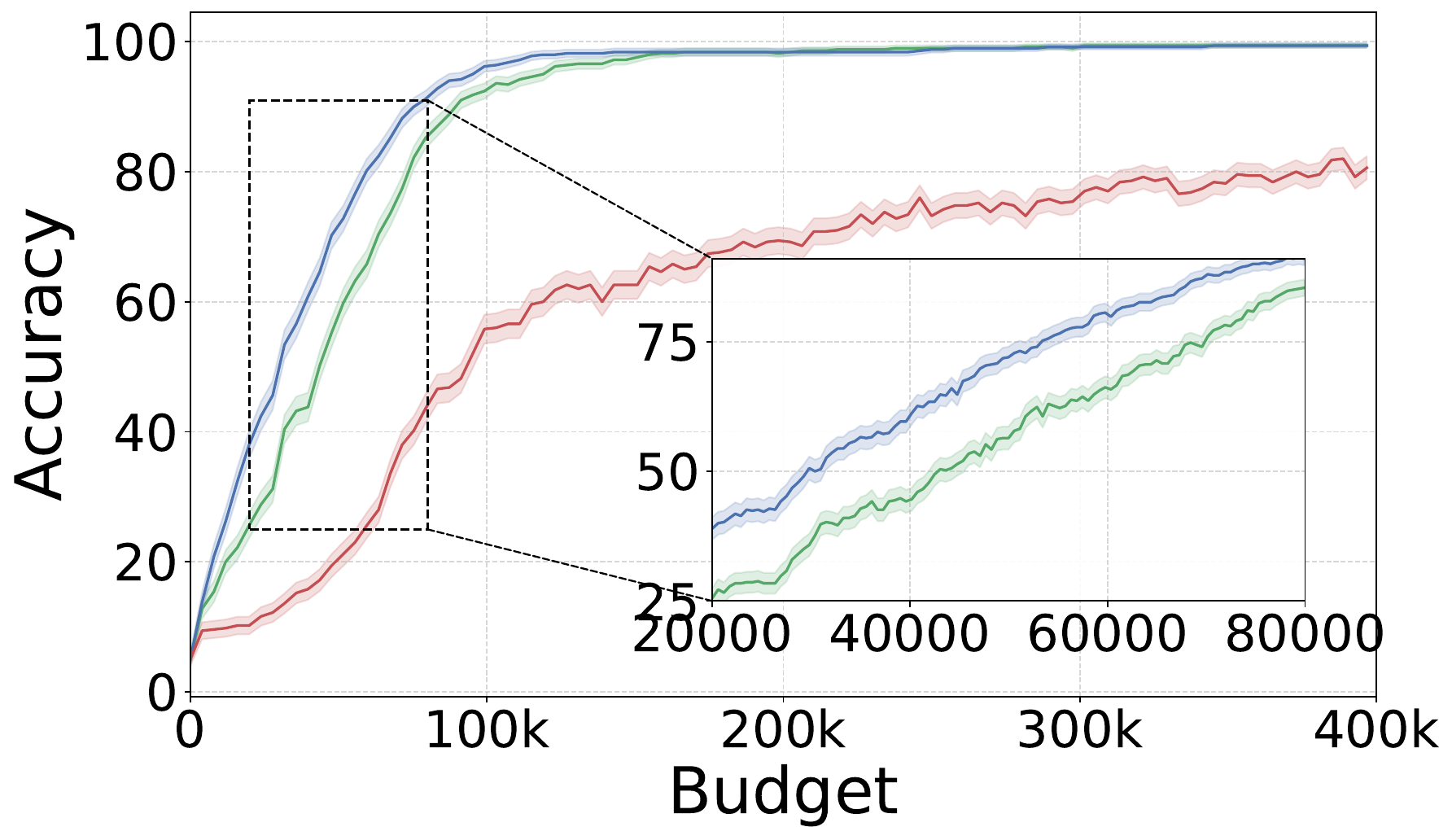}
    \subcaption{Bench 1($\Delta_{\text{top2}}= 0.01$)}
\end{subfigure}
\begin{subfigure}[t]{0.45\textwidth}
    \centering
    \includegraphics[width=\textwidth]{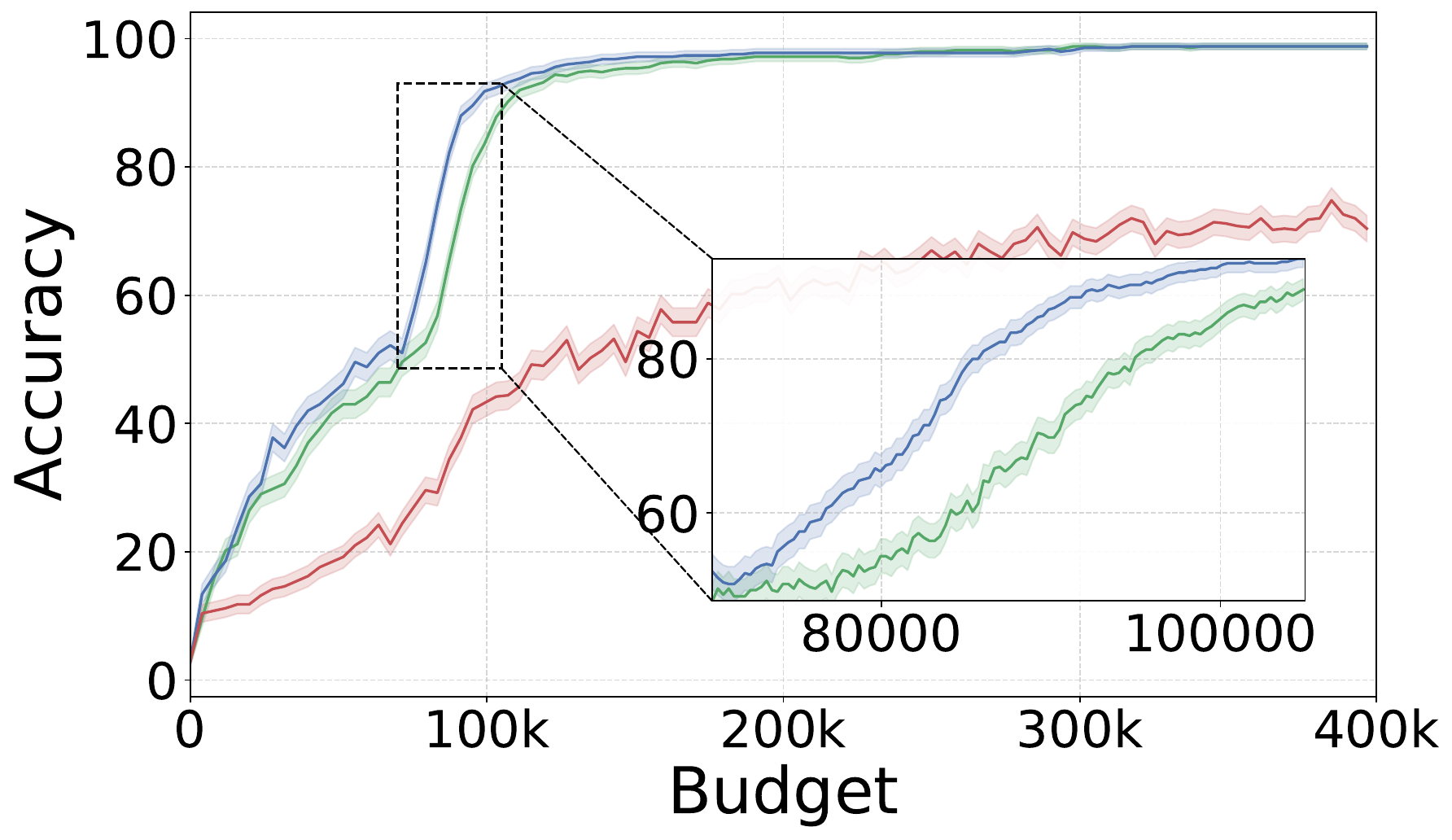}
    \subcaption{Bench 1(Full dataset)}
\end{subfigure}
\\
\begin{subfigure}[t]{0.45\textwidth}
    \centering
    \includegraphics[width=\textwidth]{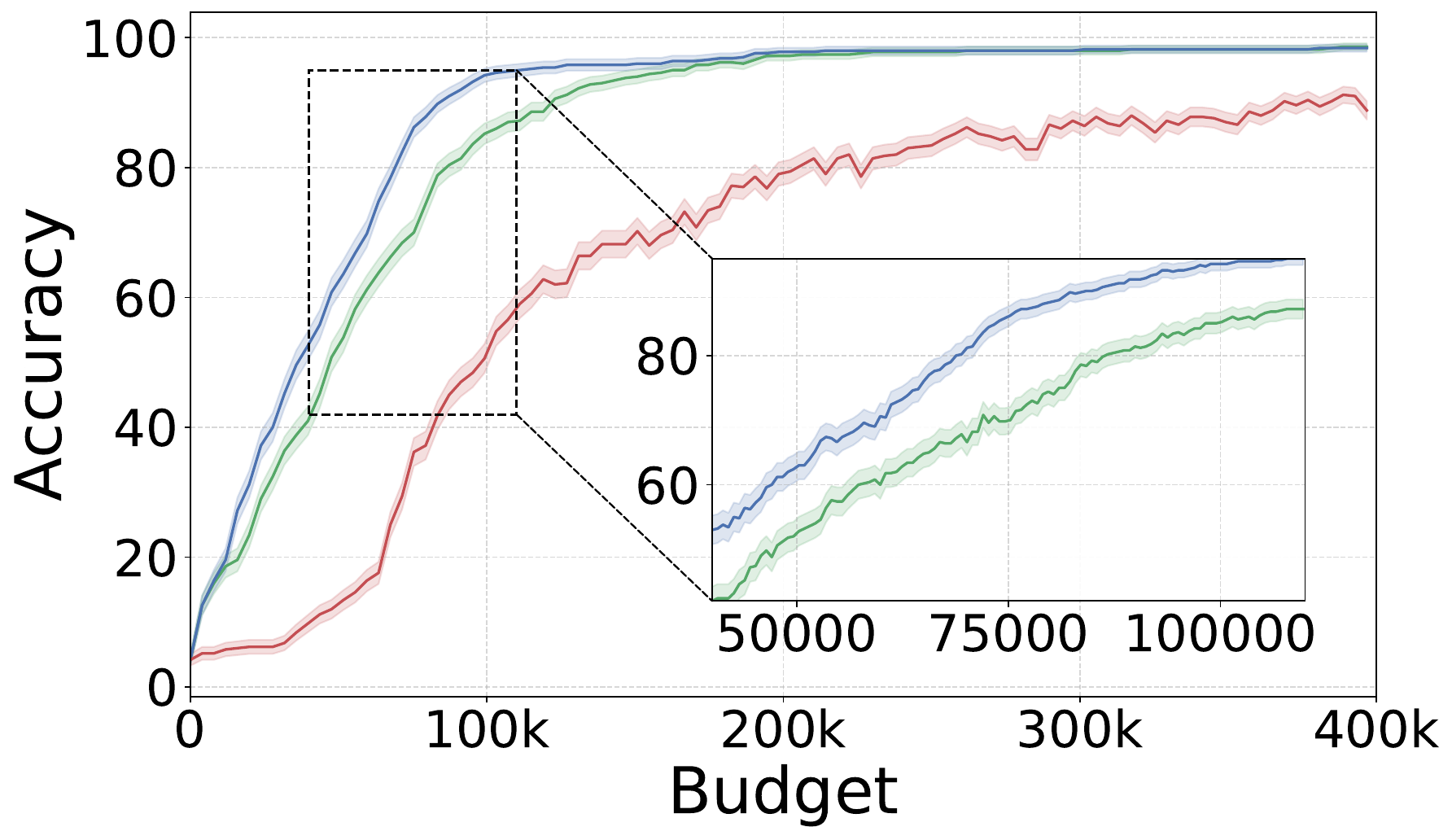}
    \subcaption{Bench 2($\Delta_{\text{top2}}= 0.01$)}
\end{subfigure}
\begin{subfigure}[t]{0.45\textwidth}
    \centering
    \includegraphics[width=\textwidth]{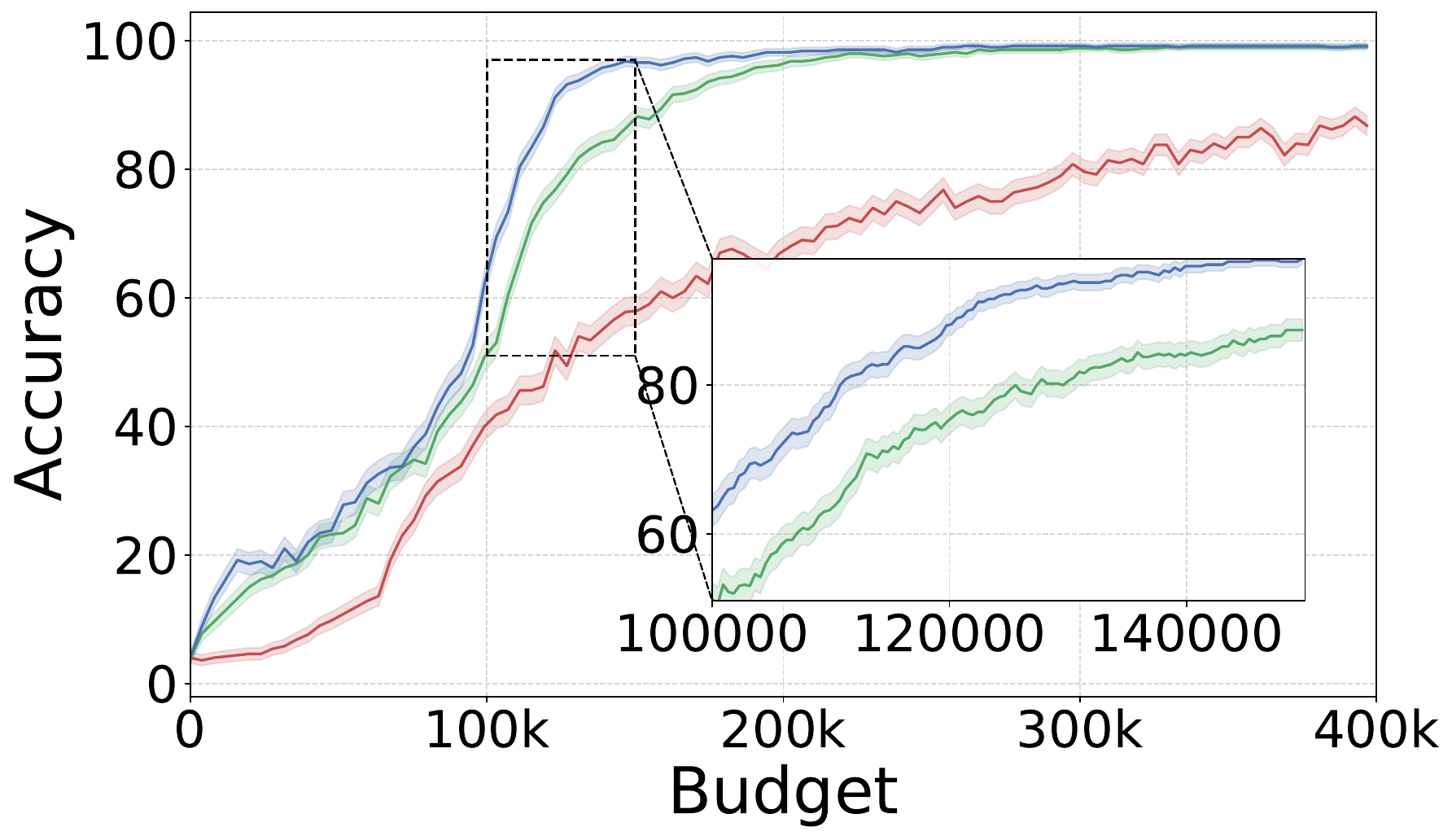}
    \subcaption{Bench 2(Full dataset)}
\end{subfigure}
\caption{Accuracy versus budget. Shaded areas around each curve depict standard error.}
\label{fig:accuracy_vs_budget_appendix}
\end{figure*}

\section{Experiment Details}
\label{appendix:experiments}
The dataset we use is taken from~\cite{wu2026efficient} and was downloaded from the paper's GitHub repository\footnote{\url{https://github.com/skbwu/efficiently-evaluating-llms/tree/main/data/processed}} at commit \texttt{857ee18}. The dataset is licensed under the Apache 2.0 license.

\subsection{Hyperparameter}
To fit the low-rank factorization model, we used the AdamW optimizer.
 We used the following parameters in our experiments:
 
 \paragraph{Bench 1.} When fitting the low-rank model on the training data we use: $r=100$, $\lambda = 0.001$, learning-rate=$0.01$, and $5$ epochs. For the initialization phase we use: $\lambda=0.01$, learning-rate=$0.01$,epochs=$50$. During the bandit loop we use learning-rate=$0.1$, $\lambda=0.01$, epochs=$10$.

 \paragraph{Bench 2.} When fitting the low-rank model on the training data we use: $r=100$, $\lambda = 0.01$, learning-rate=$0.01$, and $5$ epochs. For the initialization phase we use: $\lambda=0.01$, learning-rate=$0.01$,epochs=$50$. During the bandit loop we use learning-rate=$0.1$, $\lambda=0.01$, epochs=$10$.

All methods used $a=1$ to run the algorithm.

\subsection{Experiments compute resources}
Experiments were conducted on cpus and GPUs. We use Nvidia A40. A single run of our algorithm with budget $540,800$ took 60 seconds on average.

\clearpage
\newpage

\newpage
\end{document}